\documentclass[12pt]{article}

\usepackage{scicite}
\usepackage[flushleft]{threeparttable}
\usepackage{longtable}
\usepackage{times}
\usepackage[utf8]{inputenc}
\usepackage{amsfonts}       
\usepackage{nicefrac}       
\usepackage{microtype}      
\usepackage{amssymb,amsmath,amsthm,mathtools,bbm}
\usepackage{algorithm,algpseudocode}
\usepackage{amsthm}
\usepackage{subcaption}
\usepackage{graphicx}
\usepackage{blindtext}
\usepackage{capt-of}
\usepackage{booktabs}
\usepackage{url}
\usepackage{varwidth}
\usepackage{authblk}
\usepackage{adjustbox}
\usepackage{tabularx}
\usepackage{caption} 
\usepackage{array} 
\usepackage{xcolor}
\usepackage{rotfloat}
\usepackage[left]{lineno}
\usepackage{threeparttable}
\usepackage{makecell}
\usepackage{multirow} 
\usepackage{csquotes}
\usepackage[normalem]{ulem}
\usepackage{comment}

\newtheorem{theorem}{Theorem}

\newenvironment{sciabstract}{
\begin{quote} \bf}
{\end{quote}}

\title{Improving Access to Essential Medicines via Decision-Aware Machine Learning}

\author{Angel Tsai-Hsuan Chung,$^{1}$ Patrick Bayoh,$^{3}$ Jatu Abdulai,$^{3}$ Lawrence Sandi,$^{3}$ Francis Smart,$^{4}$ Hamsa Bastani,$^{1,\ast}$ Osbert Bastani$^{2,\ast}$\\
\normalsize{$^{1}$Department of Operations, Information, and Decisions, University of Pennsylvania,}\\
\normalsize{Jon M. Huntsman Hall, 3730 Walnut St, Philadelphia, PA 19104, USA}\\
\normalsize{$^{2}$Department of Computer and Information Science, University of Pennsylvania,}\\
\normalsize{3330 Walnut St, Philadelphia, PA 19104, USA}\\
\normalsize{$^{3}$National Medical Supplies Agency, Sierra Leone}\\
\normalsize{31 Murray Town Road, Freetown, Sierra Leone}\\
\normalsize{$^{4}$Ministry of Health and Sanitation, Sierra Leone}\\
\normalsize{Youyi Building Freetown, Freetown, Sierra Leone}\\
\normalsize{$^\ast$To whom correspondence should be addressed; E-mail:  hamsab@wharton.upenn.edu, obastani@seas.upenn.edu.}
}

\date{}

\begin{document}

\baselineskip24pt

\maketitle

\begin{sciabstract}
A critical challenge in healthcare systems in Low- and Middle-Income Countries (LMICs) is the efficient and equitable allocation of scarce resources, particularly essential medicines~\cite{yadav2015health}. This problem is complicated by limited high-quality data, which restricts the applicability of traditional data-driven techniques~\cite{unicef2015state, gallien2021inventory,merkuryeva2019demand,yenet2023challenges}. We propose a novel decision-aware machine learning framework for essential medicines allocation, which additionally leverages multi-task learning to ensure sample efficiency and catalytic priors to ensure equitable allocation. In collaboration with the Sierra Leone national government, we performed a staggered, nationwide deployment of our system as a decision support tool. Our evaluation finds an estimated 19\% increased consumption of allocated products in treated districts, demonstrating its efficacy at improving access to essential medicines. Our tool was subsequently scaled nationwide, covering an estimated 2 million women and children under five. Our work demonstrates how machine learning methods can improve efficiency at very low cost in resource-constrained global health settings.
\end{sciabstract}

\section{Introduction}

Machine learning has demonstrated enormous potential for improving healthcare, with applications ranging from screening and diagnosis~\cite{de2018clinically, esteva2019guide, ngiam2019big}, targeted testing~\cite{bastani2021efficient}, and automating medical record generation~\cite{yang2022large}. Success stories have largely occurred in developed nations, which have high-quality data available to train predictive models. Yet, machine learning has unique potential to help alleviate challenges caused by infrastructure limitations in developing nations, such as poor stock management systems~\cite{zuma2019challenges} and insufficient staff training in logistics and inventory management~\cite{gallien2021inventory, yadav2015health}. These limitations can introduce operational challenges that make it difficult to match supply to demand, leading to shortages of common medicines and medical supplies taken for granted in developed nations. An assessment by the World Health Organization (WHO) found that on average across a representative sample of health facilities in several African countries (including Sierra Leone), less than 40\% of maternal essential medicines were readily available~\cite{whosara2018,droti2019poor}. This scarcity can force vulnerable patients to pay inflated private-sector prices or forgo treatment~\cite{yenet2023challenges}. While there have been efforts to address these challenges through traditional methodologies for improving operations~\cite{tetteh2009creating, terwiesch2006matching, fenta2017human, johnson2021inventory}, recent work demonstrates that the impact of these efforts are limited by issues such as misalignment between digital systems and supply chain operations~\cite{mwashiuya2025effect} or due to insufficient technical expertise and infrastructure to manage complex supply chains~\cite{braa2007building}.

We focus on Sierra Leone, which launched the Free Health Care Initiative (FHCI) in 2010, one of the largest healthcare initiatives and a top priority in its post-civil war recovery~\cite{donnelly2011did}. This program provides free medical care and products to pregnant women and children under five. However, these essential medicines are only distributed to healthcare facilities once a quarter, which can lead to significant shortages in one location even if there is adequate supply in other locations. Efforts to design more adaptive supply chains have failed due to logistical challenges. For example, inspired by The Coca-Cola Company's highly effective supply chain, Project Last Mile (PLM) operated since 2018 with funding from the United States Agency for International Development (USAID) to build software to digitize supply chain data, with the goal of supporting more frequent and targeted restocking. However, despite operating for half a decade, their systems have been implemented in only about 8\% of facilities (15 government hospitals and 103 health facilities) as of 2023 due to logistical challenges~\cite{linnander2017process,linnander2018mixed,projectlastmile2023}.

In this context, a low-cost and promising avenue to scalably reduce shortages is to better match limited supply with patient demand using a combination of prediction and optimization~\cite{zhu2021demand, mbonyinshuti2021prediction, bastani2024optimizing}. In particular, given accurate facility-level demand forecasts for a product, existing supply chain optimization techniques~\cite{bertsimas2006robust,yang2021multilocation,aviv2001effect} optimize facility-level allocations to minimize unmet patient demand.

However, accurately forecasting demand is difficult because of its high variability, both due to demand shocks caused by natural disasters or disease outbreaks~\cite{unicef2015state, yadav2015health, yenet2023challenges} and periodic trends caused by seasonal diseases such as malaria. Existing strategies employed in global health contexts include relying on historical consumption data, morbidity-based forecasting, or proxy estimates~\cite{USAID2014HealthCommodities, Kenya2016FamilyPlanning}; however, these approaches tend to have limited accuracy due to demand variability. Furthermore, due to the lack of modern computing infrastructure, decision-makers in practice rely primarily on ad-hoc manual forecasting or on Excel spreadsheets with very high forecasting errors~\cite{merkuryeva2019demand,fredrick2016factors}.

Modern machine learning tools provide a promising path to improve the accuracy of demand forecasting~\cite{ban2019big}. By leveraging high-dimensional covariates, they can predict complex trends in demand that are beyond the reach of simpler methodologies. The key obstacle to realizing the promise of these approaches is the lack of training data. What little data is available has significant rates of missing or unreliable entries due to the same infrastructure limitations that cause operational challenges. For instance, in Sierra Leone, the only digitized data that can be used to train facility-level demand forecasting models is monthly consumption data recorded in the District Health Information Software 2 (DHIS2) system (used by 53 of the 54 countries in Africa)~\cite{dhis2}. This data has only been collected systematically across all health facilities since 2020; furthermore, it has a high rate of missing or unreliable entries. These limitations make it difficult to train accurate predictive models based on high-dimensional covariates.

In this paper, we introduce a novel machine learning framework for constrained resource allocation that satisfies two key criteria: (1) it makes effective use of limited and noisy historical data, and (2) it is scalable and can be deployed in limited-compute environments. At a high level, our system leverages machine learning to predict demand based on features constructed from historical data, and then applies stochastic optimization to compute the best allocation based on this model's predictions. To tackle data scarcity and quality challenges, our system uses a multi-task learning strategy to share data across different healthcare facilities~\cite{caruana1997multitask} along with a novel decision-aware learning algorithm~\cite{donti2017task,elmachtoub2022smart,bertsimas2020predictive,wilder2019melding, wang2020automatically, shahdecision} that aligns the prediction loss with the downstream allocation objective. Furthermore, it uses catalytic priors~\cite{huang2020catalytic} with auxiliary data sources to mitigate data inequity (i.e., where poorer facilities have lower-quality data and therefore noisier forecasts).

In close collaboration with the Sierra Leone national government, we deployed our system nationwide as a decision support tool to improve allocation decisions for Sierra Leone's FHCI policy. Our system was first piloted in 5 randomly selected districts, enabling us to econometrically evaluate our deployment. A Synthetic Differences-in-Differences analysis~\cite{arkhangelsky2021synthetic} found that consumption increased by 19\% in the 5 treated districts, suggesting that our system significantly improved patient access to essential medicines it helped allocate. Following the pilot, our system has been scaled nationwide, covering an estimated 2 million women and children under five~\cite{WorldBank2025,UNICEF_2025_SierraLeoneData}.
Importantly, these gains were extremely cost-effective, requiring only a \$30 monthly server fee and no additional workforce (see Supplement~\S\ref{sup:cost} for a cost-effectiveness analysis). These results highlight how novel data-efficient machine learning can address critical challenges in healthcare delivery in resource-constrained settings.

\section{Essential Medicines Allocation System}
\label{allocation}

Maternal mortality remains one of the most pressing global health challenges, particularly in developing countries where access to essential health products is often limited by inefficient allocation systems~\cite{souza2024global}. In Sierra Leone, despite the FHCI providing free medical care to pregnant women and children, maternal mortality rate stands at 717 per 100,000 live births --- one of the highest in the world~\cite{who_maternal_mortality, shafiq2024causes}.

Lack of access to essential medicines is one of the key contributors to preventable maternal deaths~\cite{souza2024global}. In Sierra Leone, the National Medical Supplies Agency (NMSA) (part of the Ministry of Health and Sanitation (MoHS)) was created to manage the procurement and distribution of medicines and medical supplies to public health facilities across the nation, including over 70 essential medicines for women and children. Supply of these medicines largely relies on donations from international organizations. Prior to our collaboration, the NMSA distributed these medicines across the country following a centralized two-level push system~\cite{terwiesch2006matching}, where the NMSA first allocated supply to 16 districts, which districts then allocate to individual health facilities in their catchment (see details in Supplement \S\ref{sup:ExistingAllocation}). The two-level system is for administrative purposes only; all warehouses and health facilities are owned and operated by the NMSA. Allocation decisions were largely made via an Excel tool; however, it was unable to capture key demand patterns such as seasonality and other time-varying fluctuations. As a result, district pharmacists often made significant manual adjustments because they perceived that the tool’s estimates did not accurately reflect the actual needs of health facilities. According to NMSA officials, approximately 42\% of facility requests were unfulfilled on average.

In practice, this shortage occurred even when the total supply was adequate---some health facilities received a surplus of medicines while others suffered shortages (surpluses often did not carry over fully to the next quarter due to significant reported waste and expiration). For instance, DHIS2 data showed significant heterogeneity in both stockouts and wastage---e.g., for the Tonkolili district, in 2022 Q1, 10\% of facilities that had previous stockouts continued to face supply shortages, while 22\% of facilities with available stock had excess supply that could have been redistributed. These findings suggest more efficient allocation strategies that better match supply and demand as a promising path to significantly reducing shortages.

\begin{figure}[t]
\begin{center}
\includegraphics[width=0.95\linewidth]{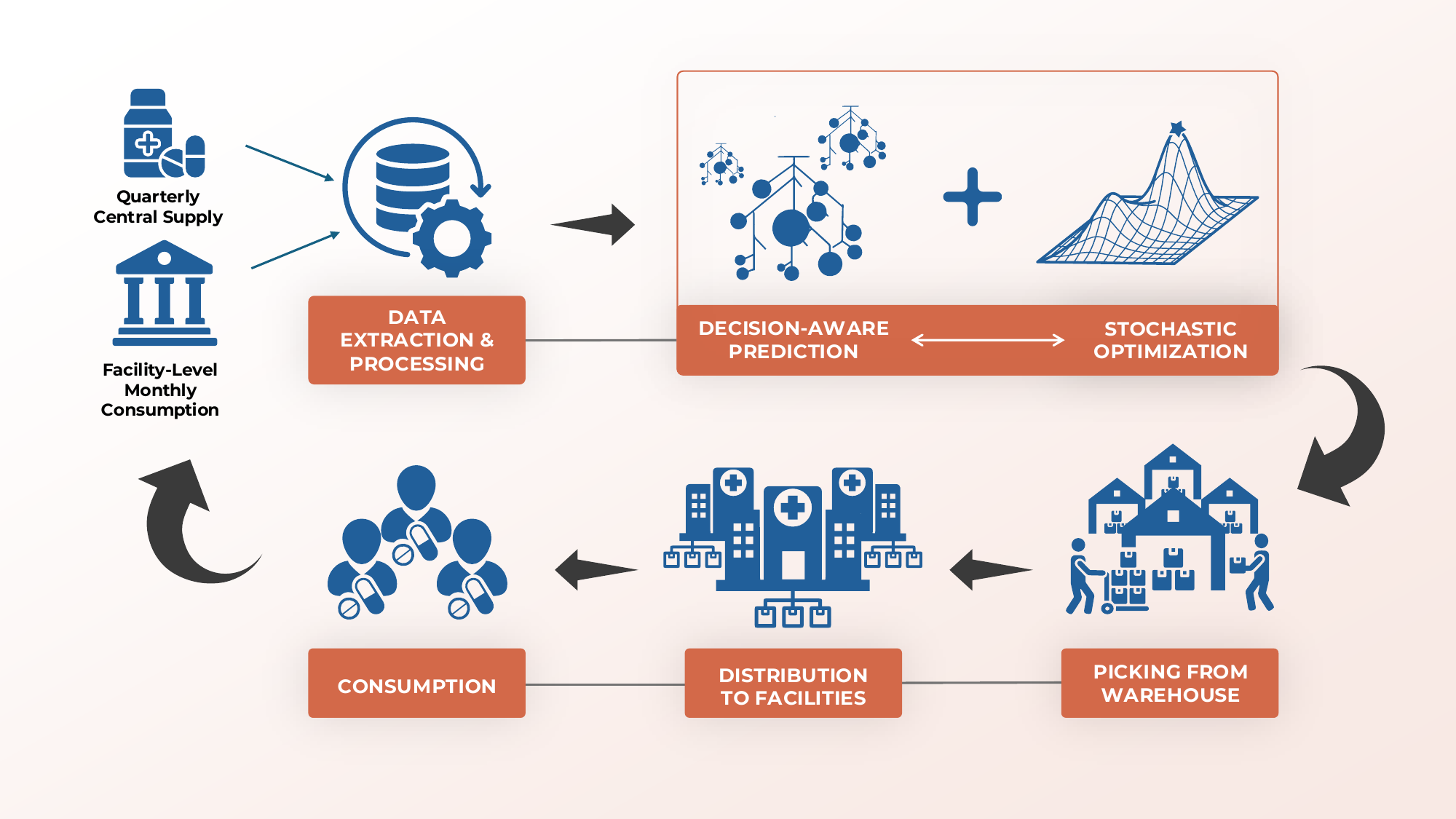}
\caption{\textbf{System Overview.} Every quarter, our system extracts and processes data including the total available central stock and historical facility-level consumption records from mSupply and the DHIS2 government database. It then trains a decision-aware prediction model that informs a stochastic optimization procedure to make allocation decisions. The system further provides a picking list for frontline workers to collect supplies from designated warehouses, enabling efficient distribution to local health facilities. Finally, the resulting patient consumption data is recorded for informing predictions and decisions in subsequent quarters.}
\label{fig:mlframework}
\end{center}
\end{figure}

Fig.~\ref{fig:mlframework} illustrates the system we developed and deployed. Unlike the earlier two-level approach, our system directly generates facility-level allocations, which are provided to district pharmacists to support their decision-making. Our system begins by pulling monthly consumption data from DHIS2, along with quarterly data on total central stock available for allocation and their expiry dates from the mSupply warehouse management system (see Supplement~\S\ref{sup:datasets} for details). Then, it performs significant preprocessing to ensure data reliability and construct informative features for prediction (see Section~\ref{sup:dataprocessing} for details).

Our prediction and optimization algorithm (illustrated in the box in Fig.~\ref{fig:mlframework}) can be divided into two components. First, we predict the demand distribution for each facility-product pair using a novel decision-aware machine learning framework (described in the next section). Second, given the demand forecasts, our optimization algorithm outputs allocation decisions directly from the central stock to individual health facilities, which are designed to minimize the expected shortage of medicines. In particular, for each product, it aims to minimize \emph{unmet demand}---the total number of units of requested product that go unfulfilled---across facilities. To account for the stochastic nature of demand, we use the expected unmet demand according to probabilistic demand forecasts. This optimization problem can be solved efficiently via linear programming using a sample average approximation~\cite{birge2011introduction}. Once the allocation is determined, our system assigns each batch of supplies to a warehouse based on proximity, central stock availability, and product expiration dates (see Section~\ref{sup:warehouse} for details). Finally, new patient consumption data is recorded and used to retrain our model each quarter just before making allocation decisions.

\section{Machine Learning for Demand Forecasting}

Next, we describe our machine learning framework for predicting the demand distribution used in our optimization algorithm. Given a training dataset constructed from historical demand data, we could apply a traditional strategy for time series forecasting such as Autoregressive Integrated Moving Average (ARIMA)~\cite{gilbert2005arima}. However, these approaches work poorly due to data scarcity---e.g., when we trained our model for 2023 Q2 allocations, each facility-product pair had at most 37 monthly observations (i.e., from January 2020 to January 2023).

Instead, we designed a machine learning framework that leverages three key techniques: (1) multi-task learning to share data across facilities, (2) catalytic priors to regularize the model in data-poor regions to mitigate data inequity, and (3) a novel decision-aware learning algorithm to focus predictive power on facilities that are most relevant to the downstream optimization problem. We summarize our techniques below, and provide details in Section~\ref{sup:mlframework}.

We developed a multi-task learning strategy that trains one demand prediction model for each of our two product types (medicines and medical supplies/equipment) across all facility-product pairs within that group. Our multi-task learning strategy enables knowledge transfer from locations with more available data to ones with less available data~\cite{caruana1997multitask,bastani2021predicting,xu2025multitask}. Specifically, our system constructs features that facilitate generalization across facility-product pairs (e.g., average demand in the past year), as well as features that capture trends specific to a given facility-product pair (e.g., the facility type and location, product fixed effect). Then, it trains a random forest to predict demand from these features (we use a random forest since it consistently outperforms other methods; see Supplement~\S\ref{sup:outofsample}).

While multi-task learning can improve performance in data-poor locations, there are still systematic differences between data-poor and data-rich locations (e.g., missing data often arises disproportionately in poorer regions due to staffing shortages). Such \emph{missing-not-at-random} data~\cite{rubin1976inference} can lead to covariate shift, which reduces prediction accuracy for data-poor locations. To mitigate this data inequity, our system leverages catalytic priors~\cite{huang2020catalytic} to regularize predictions for data-poor locations towards a simpler, less biased population-based model. Specifically, it first uses census and satellite data to estimate catchment population, and then estimates demand proportionally to the catchment population. This strategy ensures complete coverage of all facilities without biases due to low-quality data. Our system integrates this simple model as a catalytic prior for our machine learning model, regularizing predictions for data-poor locations towards the population-based ones. 

Finally, our system leverages decision-aware learning~\cite{wilder2019melding,bertsimas2020predictive,kallus2022stochastic,shahdecision} to focus predictive power on instances most relevant to optimizing allocation. Intuitively, not all predictions matter equally---e.g., since our goal is to prevent unmet demand, we are primarily interested in how much we should stock facilities that are \textit{likely} to be insufficiently stocked for a particular product. Whereas a standard machine learning approach would treat all facility-product pairs equally when training a prediction model, our decision-aware learning algorithm focuses attention on facility-product pairs deemed more important to the decision-making objective; in this case, it upweights observations corresponding to facilities that are likely to be under-stocked. We found that existing decision-aware learning algorithms were either computationally intractable at our scale or incompatible with the rest of our prediction and optimization pipeline; thus, we developed a novel decision-aware learning algorithm that could be integrated more easily. Prior to deploying our framework, we validated it on historical data, showing that it outperformed existing approaches on a held-out test set; see Supplement~\S\ref{sup:outofsample}.

\section{Deployment}

In collaboration with the NMSA,  in May 2023, we deployed our system in 5 of the 16 districts in Sierra Leone to perform facility-level allocations for the second quarter (June through August). These districts---Tonkolili, Falaba, Karene, Kono, and Pujehun (see map in Fig.~\ref{fig:treatmentmap})---were selected by the central government based on a randomized allocation schedule. Prior to our deployment, the government had already established both the overall supply as well as the total amounts to be allocated to control districts. The remaining supply was then to be allocated to the treatment districts, maintaining independence of supply quantities between the two groups. Additional deployment details are in Supplement~\S\ref{sup:deployment}.

Beyond the performance of the allocation system, a critical priority was ensuring that it could be seamlessly integrated into NMSA's existing workflows. To this end, we conducted two targeted training sessions for policymakers and frontline workers, providing the necessary technical knowledge for operating our tool and understanding its implications. Implementation emphasized stakeholder engagement at every level---outputs were presented in a familiar format that matched pre-existing workflows and reviewed by NMSA management as well as district pharmaceutical managers prior to finalization.

Furthermore, our system functions as a decision support tool, where central and district-level planners retain ultimate authority and can override the system when necessary. We found that decision-makers closely followed the algorithmic allocations---the normalized overlap between the actual and algorithmic allocations ranged from 0.89 to 1 (see Supplement \S\ref{sup:compliance} for details). Given the high compliance rates, stakeholder buy-in, and early indications of improved efficiency (detailed in the following section), the national government adopted our system and expanded its use to all public-sector health facilities in Sierra Leone beginning in the third quarter of 2023.

\begin{figure}[t]
\centering
\includegraphics[width=.5\columnwidth]{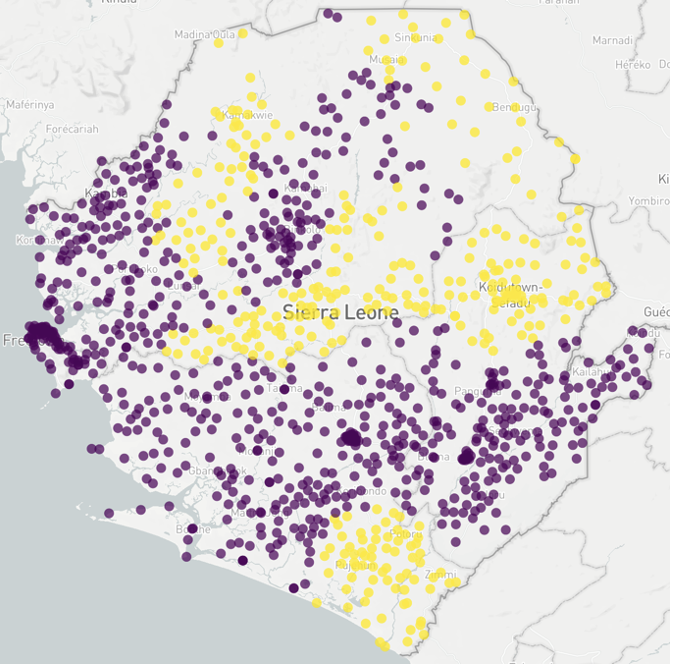}
\caption{\textbf{Map of treatment distribution in 2023 Q2.} Yellow dots denote treated facilities (i.e., those in the Tonkolili, Falaba, Karene, Kono, and Pujehun districts), and purple dots denote control facilities (i.e., those in the Kailahun, Kenema, Bombali, Koinadugu, Kambia, Port Loko, Bo, Bonthe, Moyamba, Western Area Rural, and Western Area Urban districts).}
\label{fig:treatmentmap}
\end{figure}

\section{Evaluation}

We evaluated the effectiveness of our system at improving allocation efficiency. While our optimization objective aims to minimize unmet demand, we do not directly observe this quantity (we only observe when patients receive medicine, not when they are turned away due to stockouts). Instead, examined the equivalent objective of maximizing \textit{patient consumption}, which is directly observed---since patient demand (which is fixed but unobserved) equals consumption plus unmet demand, maximizing total consumption is equivalent to minimizing total unmet demand. (This equivalence holds in a very general setting under the assumption that facilities do not wastefully give away medicines to patients that do not need them, which is unlikely in a such a resource-constrained environment; we provide theoretical and empirical support for this equivalence in Supplement~\S\ref{sup:consumption}.) We overview our analyses here, and provide details in Section~\ref{sup:mainanalysis} and Supplement~\S\ref{sup:evaldeploy}; our results are summarized in Fig.~\ref{fig:combined} and Table~\ref{tab:att_tabular}.

Our main analysis focuses on the second quarter of 2023 (2023 Q2), where our system was deployed in a randomly selected subset of districts, providing natural variation enabling us to estimate causal effects. Specifically, the NMSA performed allocation using standard procedures for the 11 control districts, after which our system was used to distribute remaining stock across the 5 treated districts. We show descriptive time series trends of average normalized consumption for the treatment and control groups in Fig.~\ref{fig:combined} (left). We estimated how much our system changed patient consumption levels in treated districts (i.e., districts where our system was deployed) compared to what would have happened without our intervention, known as the Average Treatment Effect on the Treated (ATT). We used a balanced panel dataset of 312 facilities in treated districts and 746 facilities in control districts using time series data beginning in 2022 Q3 through 2023 Q3, after which our tool was used nationwide. (Prior to
the implementation of the data-driven Excel allocation tool by Crown Agents in 2022 Q3, allocation procedures were highly inconsistent, rendering the data unreliable.)

Given the limited number of treated districts, we used a Synthetic Difference-in-Differences (SynthDiD) analysis~\cite{arkhangelsky2021synthetic} to analyze the impact of our deployment. SynthDiD combines the strengths of Difference-in-Differences~\cite{card1993minimum} (which assumes similar trends between treated and untreated groups in the pre-treated period) and Synthetic Controls~\cite{abadie2010synthetic}
(which constructs a synthetic control group that is similar to the treated units in terms of observed characteristics and outcome trends in the pre-treated period). SynthDiD ensures that differences between facilities in treated districts and synthetic control facilities remain stable prior to treatment.

Fig.~\ref{fig:combined} (right) shows the corresponding time series trends for the treatment and synthetic control groups, and results are shown in the ``SynthDiD'' row of Table~\ref{tab:att_tabular}. The 5 treated districts experienced a statistically significant increase of 19\% ($p<0.01$) in consumption; the improvement percentage is calculated from SynthDiD counterfactual estimates as:
\begin{align*}
\frac{\text{Average treatment outcome} - \text{Average counterfactual outcome}}{\text{Average counterfactual outcome}}.
\end{align*}
We validated our SynthDiD approach using a standard event-study analysis~\cite{clarke2023synthetic}, which showed that there are no statistically significant differences between treated and synthetic control facilities prior to our intervention, and that the change in consumption emerged only after our system was deployed. These results suggest that our system substantially improved access to essential medicines in treated districts. We provide additional details on our main analysis in Section~\ref{sup:mainanalysis}.

\begin{figure}[t]
\centering
\begin{subfigure}[t]{0.49\linewidth}
\centering
\includegraphics[width=\linewidth]{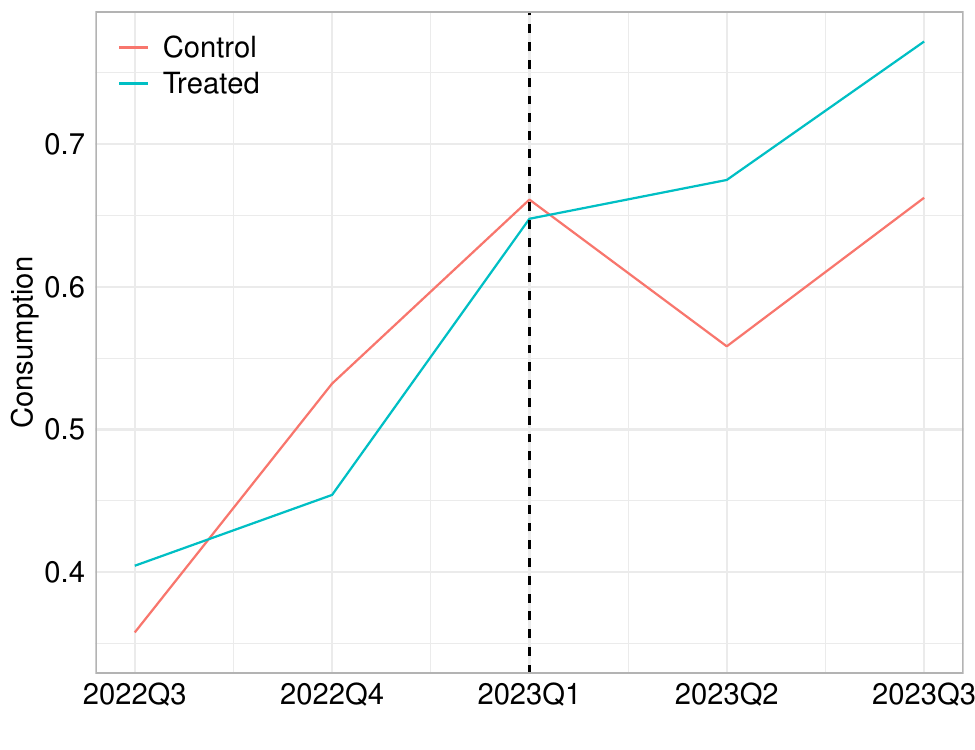}
\caption{\textbf{Treatment vs Control Groups.}}
\label{fig:raw}
\end{subfigure}
\hfill
\begin{subfigure}[t]{0.49\linewidth}
\centering
\includegraphics[width=\linewidth]{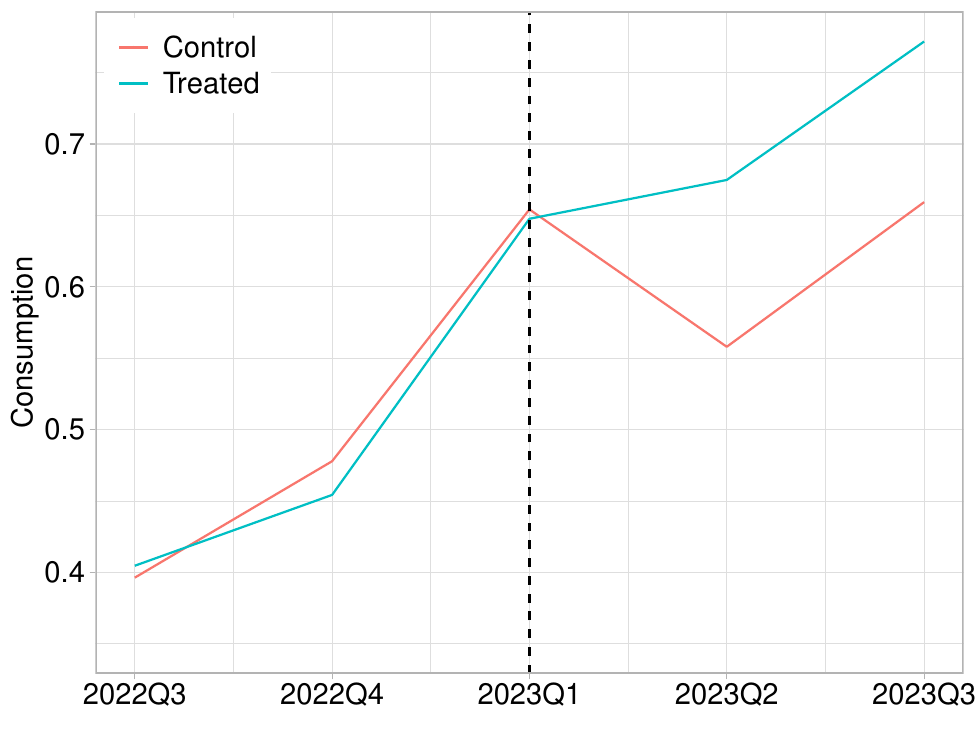}
\caption{\textbf{Treatment vs Synthetic Control Groups.}}
\label{fig:Main}
\end{subfigure}
\caption{\textbf{Average normalized consumption time trends.} Panel (a) compares treatment (green) versus control (orange) groups using raw data. Panel (b) compares treatment (green) versus synthetic control (orange) groups constructed via SynthDiD. The $x$-axis shows time in quarters and the $y$-axis denotes average normalized consumption. The vertical dashed line at 2023 Q1 marks the quarter immediately before deployment.}
\label{fig:combined}
\end{figure}

Next, we studied how these improvements were distributed; we summarize our analyses here and provide details in Supplement~\S\ref{sup:heteroanalysis}. First, the NMSA seeks to prioritize larger facilities such as hospitals, since these facilities are most relied upon to provide quality care; indeed, we found that our system confers systematically greater benefits (36\%, $p<0.05$) for larger facilities, while preserving consumption for smaller ones. In contrast, smaller health facilities showed a smaller, statistically insignificant improvement. They may face more fundamental challenges (e.g., inadequate staffing or equipment) that limit their ability to provide services or utilize medical supplies regardless of availability---for instance, during a field visit, we found a small health post closed for two consecutive days, and local residents shared that they often seek essential care at larger, more reliable facilities.

Next, we sought to ensure our consumption gains are not concentrated among wealthier patients, but also benefit poorer patients with the greatest need. To this end, we examined previously under-served facilities (i.e., facilities that experienced at least one stockout in our data prior to our deployment). We found that these facilities also saw a significant increase in consumption (32\%, $p<0.01$), suggesting that our system successfully addressed potential biases from uneven data quality and availability, leading to more equitable resource allocation. Similarly, we found that rural facilities (which tend to serve poorer populations) benefited from our tool without compromising urban ones.

Finally, we performed several robustness checks to validate our main results. The main checks are described below and summarized in Table~\ref{tab:att_tabular}; Supplement~\S\ref{sup:robustnesschecks} provides details as well as additional checks.  First, to ensure consistency under a simpler methodology, we performed a standard DiD analysis, finding a consistent 21\% increase in consumption ($p<0.01$) (``DiD'' row in Table~\ref{tab:att_tabular}). Second, to better account for geographic factors, we performed a DiD analysis on geographically matched facility pairs (i.e., facilities within a short distance of a district border, which have a neighboring facility with the opposite treatment status), again yielding a consistent 21\% increase in consumption ($p<0.05$) (``Matching'' row).

Third, we considered an alternative analysis where we compared consumption trends across products instead of facilities; this robustness check uses a substantially different design with additional data to build confidence in our results. In particular, we considered 25 other products that were concurrently allocated using a different, pre-existing mechanism.
The consumption levels for these products can be used as a control group throughout our study period across \textit{all} districts in the country using a staggered treatment---i.e., an advantage of this analysis is that it can be performed not just for the partial deployment in 2023 Q2, but also for the nationwide implementation starting in Q3. We found a consistent, statistically significant increase in consumption of 18\% ($p<0.001$) (``Alt. Control'' row).

Next, a potential concern with our evaluation is due to missing data from facilities that failed to record consumption in a particular month, which may bias our results. Our main analysis dropped observations with missing outcomes, but we also performed robustness checks using multiple standard imputation strategies---low-rank matrix completion~\cite{candes2012exact}, population-based methods, and historical average consumption. The resulting SynthDiD ATT estimates are all statistically significant and consistent with our main analysis (``Imputation'' rows). We also performed additional tests suggesting that our results are not driven by differential missingness; see Supplement~\S\ref{ssec:missingness} for details.

\begin{table}[t]
\centering
\caption{\textbf{Average Treatment Effect on the Treated (ATT).} ATT estimates across different estimation strategies. SynthDiD, our main specification, shows a 19\% increase in consumption, with similar magnitudes under DiD and several imputation approaches. Estimation using alternative control product data or geographic facility-level matching across district borders also produce consistent positive effects. For stockouts as the outcome, the point estimate is negative but statistically insignificant; this is expected since it is not our objective.}
\label{tab:att_tabular}
\begin{threeparttable}
\begin{tabular}{l c c c r}
\toprule
\textbf{Model} & \textbf{Coef.} & \textbf{Std. Error} & \textbf{Obs.} & \textbf{ Improv.\%} \\
\midrule
\multicolumn{5}{l}{\textit{Dependent Variable = Consumption}} \\
\midrule
SynthDiD & 0.116$^{*}$ & (0.046) & 5,290 & 19\% \\
DiD & 0.128$^{**}$ & (0.046) & 5,290 & 21\% \\
Matching (15km) & 0.121$^{**}$ & (0.054) & 2,415 & 21\% \\
Alt. Control & 0.095$^{***}$ & (0.022) & 10,520 & 18\% \\
Imputation (low rank) & 0.037$^{**}$ & (0.014) & 5,455 & 15\% \\
Imputation (avg cons) & 0.067$^{**}$ & (0.031) & 5,455 & 21\% \\
Imputation (pop) & 0.076$^{**}$ & (0.028) & 5,455 & 27\% \\
\midrule
\multicolumn{5}{l}{\textit{Dependent Variable = Stockout}} \\
\midrule
SynthDiD & $-$0.280 & (0.179) & 5,290 & $-$4.6\% (Insig) \\
\bottomrule
\end{tabular}
\begin{tablenotes}[flushleft]
\scriptsize 
\item \textit{Notes:} $^{*} p < 0.05$, $^{**} p < 0.01$, $^{***} p < 0.001$. Standard errors in parentheses. Improvement \% is relative to counterfactual mean.
\end{tablenotes}
\end{threeparttable}
\end{table}

Finally, we examined the impact of our system on stockouts. While reducing stockouts might seem like a natural objective, optimizing for fewer stockouts can actually produce highly undesirable allocations---e.g., an optimal strategy is to allocate zero supply to a small number of high-volume facilities, thereby ensuring that the remaining facilities are well-stocked. We found a directional reduction in stockouts but it is not statistically significant ($p\approx0.12$) (``Stockouts'' rows)---i.e., our system does not inadvertently increase stockouts.

\section{Discussion}

Our findings provide strong field evidence of the effectiveness of our novel machine learning framework for resource allocation, significantly and equitably improving access to essential medicines in a highly constrained environment. By replacing manual effort with data-driven decision making, our system streamlined supply chain management, reduced administrative burdens, and adapted to real-time changing consumption patterns, enabling facilities to better align limited supply with patient demand. Efficient allocation becomes increasingly important as aid becomes more constrained to ensure the limited available supply meets demand.

To ensure sustainable impact, we developed a web application with an intuitive user interface, which is now owned by the Sierra Leone government. This system integrates with their government databases to automate the entire process, from data extraction and processing to generating final allocation results. Notably, it is highly cost-effective, operating without any additional workforce requirements and costing only \$30 per month nationally in server fees.

Beyond Sierra Leone, we believe our framework can be easily adapted to other contexts---by design, it relies on commonly available data: total available central stock, monthly health facility consumption, and stock expiry dates. This data often already exists in a unified, digitized format---e.g., 53 (out of 54 total) African countries use DHIS2~\cite{dhis2} and 16 use mSupply~\cite{mSupply}. The key tasks required would be adapting our system with local data collection systems and discussing our optimization objective with policy makers to incorporate country-specific constraints.

\bibliography{scibib}

@article{johnson2021inventory,
  title={Inventory management practices and supply chain performance of antiretroviral medicines in public hospitals in Nyamira County, Kenya},
  author={Johnson, Anyona and Peter, Karimi and Shital, Maru},
  journal={Rwanda Journal of Medicine and Health Sciences},
  volume={4},
  number={2},
  pages={257--268},
  year={2021}
}

@article{gallien2021inventory,
  title={Inventory policies for pharmaceutical distribution in Zambia: Improving availability and access equity},
  author={Gallien, J{\'e}r{\'e}mie and Leung, Ngai-Hang Z and Yadav, Prashant},
  journal={Production and Operations Management},
  volume={30},
  number={12},
  pages={4501--4521},
  year={2021},
  publisher={Sage Publications Sage CA: Los Angeles, CA}
}

@article{mwashiuya2025effect,
  title={Effect of Inventory Control on Medical Supply Chain Performance: The Moderating Role of Information Technology among Medical Stores Department Customers in Tanzania},
  author={Mwashiuya, Stephano and Mchopa, Alban D and Shayo, France A},
  journal={African Journal of Procurement, Logistics \& Supply Chain Management},
  volume={8},
  number={3},
  pages={01--14},
  year={2025}
}

@article{fenta2017human,
  title={Human resources for public health supply chain management in Ethiopia: Competency mapping and training needs},
  author={Fenta, Teferi Gedif},
  journal={Ethiopian Journal of Health Development},
  volume={31},
  number={4},
  pages={266--275},
  year={2017}
}

@book{terwiesch2006matching,
  title={Matching supply with demand: an introduction to operations management},
  author={Terwiesch, Christian and Cachon, Gerard},
  year={2006},
  publisher={McGraw-Hill}
}

@article{souza2024global,
  title={A global analysis of the determinants of maternal health and transitions in maternal mortality},
  author={Souza, Jo{\~a}o Paulo and Day, Louise Tina and Rezende-Gomes, Ana Clara and Zhang, Jun and Mori, Rintaro and Baguiya, Adama and Jayaratne, Kapila and Osoti, Alfred and Vogel, Joshua P and Campbell, Oona and others},
  journal={The Lancet Global Health},
  volume={12},
  number={2},
  pages={e306--e316},
  year={2024},
  publisher={Elsevier}
}

@article{besbes2013implications,
  title={On implications of demand censoring in the newsvendor problem},
  author={Besbes, Omar and Muharremoglu, Alp},
  journal={Management Science},
  volume={59},
  number={6},
  pages={1407--1424},
  year={2013},
  publisher={INFORMS}
}

@article{donnelly2011did,
  title={How did Sierra Leone provide free health care?},
  author={Donnelly, John},
  journal={The Lancet},
  volume={377},
  number={9775},
  pages={1393--1396},
  year={2011},
  publisher={Elsevier}
}

@article{greene2000econometric,
  title={Econometric analysis 4th edition},
  author={Greene, William H},
  journal={International edition, New Jersey: Prentice Hall},
  pages={201--215},
  year={2000}
}

@article{gilbert2005arima,
  title={An ARIMA supply chain model},
  author={Gilbert, Kenneth},
  journal={Management Science},
  volume={51},
  number={2},
  pages={305--310},
  year={2005},
  publisher={INFORMS}
}

@misc{who_maternal_mortality, author = {{World Health Organization}}, title = {Maternal Mortality: Fact Sheet}, year = 2024, url = {https://www.who.int/news-room/fact-sheets/detail/maternal-mortality}}

@article{shafiq2024causes,
  title={Causes of maternal deaths in Sierra Leone from 2016 to 2019: analysis of districts’ maternal death surveillance and response data},
  author={Shafiq, Yasir and Caviglia, Marta and Bah, Zainab Juheh and Tognon, Francesca and Orsi, Michele and Kamara, Abibatu K and Claudia, Caracciolo and Moses, Francis and Manenti, Fabio and Barone-Adesi, Francesco and others},
  journal={BMJ open},
  volume={14},
  number={1},
  pages={e076256},
  year={2024},
  publisher={British Medical Journal Publishing Group}
}

@book{birge2011introduction,
  title={Introduction to stochastic programming},
  author={Birge, John R and Louveaux, Francois},
  year={2011},
  publisher={Springer Science \& Business Media}
}

@article{esteva2019guide,
  title={A guide to deep learning in healthcare},
  author={Esteva, Andre and Robicquet, Alexandre and Ramsundar, Bharath and Kuleshov, Volodymyr and DePristo, Mark and Chou, Katherine and Cui, Claire and Corrado, Greg and Thrun, Sebastian and Dean, Jeff},
  journal={Nature medicine},
  volume={25},
  number={1},
  pages={24--29},
  year={2019},
  publisher={Nature Publishing Group US New York}
}

@article{ngiam2019big,
  title={Big data and machine learning algorithms for health-care delivery},
  author={Ngiam, Kee Yuan and Khor, Wei},
  journal={The Lancet Oncology},
  volume={20},
  number={5},
  pages={e262--e273},
  year={2019},
  publisher={Elsevier}
}

@article{de2018clinically,
  title={Clinically applicable deep learning for diagnosis and referral in retinal disease},
  author={De Fauw, Jeffrey and Ledsam, Joseph R and Romera-Paredes, Bernardino and Nikolov, Stanislav and Tomasev, Nenad and Blackwell, Sam and Askham, Harry and Glorot, Xavier and O’Donoghue, Brendan and Visentin, Daniel and others},
  journal={Nature medicine},
  volume={24},
  number={9},
  pages={1342--1350},
  year={2018},
  publisher={Nature Publishing Group}
}

@article{yang2022large,
  title={A large language model for electronic health records},
  author={Yang, Xi and Chen, Aokun and PourNejatian, Nima and Shin, Hoo Chang and Smith, Kaleb E and Parisien, Christopher and Compas, Colin and Martin, Cheryl and Costa, Anthony B and Flores, Mona G and others},
  journal={NPJ digital medicine},
  volume={5},
  number={1},
  pages={194},
  year={2022},
  publisher={Nature Publishing Group UK London}
}

@misc{SierraLeoneMoHS_CIESIN_2023,
  author       = {{Sierra Leone Ministry of Health and Sanitation (MoHS)} and {Center for International Earth Science Information Network (CIESIN), Columbia University}},
  title        = {Sierra Leone National Health Facilities Dataset Version 01},
  year         = {2023},
  publisher    = {Geo-Referenced Infrastructure and Demographic Data for Development (GRID3)},
  address      = {Palisades, NY},
  url          = {https://doi.org/10.7916/ag19-s334},
}

@article{weiss2020global,
  title={Global maps of travel time to healthcare facilities},
  author={Weiss, DJ and Nelson, Andy and Vargas-Ruiz, CA and Gligori{\'c}, K and Bavadekar, S and Gabrilovich, Evgeniy and Bertozzi-Villa, A and Rozier, J and Gibson, HS and Shekel, T and others},
  journal={Nature medicine},
  volume={26},
  number={12},
  pages={1835--1838},
  year={2020},
  publisher={Nature Publishing Group US New York}
}

@article{martin2020implementing,
  title={Implementing nationwide facility-based electronic disease surveillance in Sierra Leone: lessons learned},
  author={Martin, Daniel W and Sloan, Michelle L and Gleason, Brigette L and de Wit, Les and Vandi, Mohamed Alex and Kargbo, David K and Clemens, Nelson and Kamara, Ansumana and Njuguna, Charles and Sesay, Stephen and others},
  journal={Health security},
  volume={18},
  number={S1},
  pages={S--72},
  year={2020},
  publisher={Mary Ann Liebert, Inc., publishers 140 Huguenot Street, 3rd Floor New~…}
}

@article{merkuryeva2019demand,
  title={Demand forecasting in pharmaceutical supply chains: A case study},
  author={Merkuryeva, Galina and Valberga, Aija and Smirnov, Alexander},
  journal={Procedia Computer Science},
  volume={149},
  pages={3--10},
  year={2019},
  publisher={Elsevier}
}

@book{unicef2015state,
  title={The State of the World's Children 2015: Reimagine the Future-Innovation for Every Child (Executive Summary)},
  author={UNICEF and others},
  year={2015},
  publisher={UN}
}

@article{zhu2021demand,
  title={Demand forecasting with supply-chain information and machine learning: Evidence in the pharmaceutical industry},
  author={Zhu, Xiaodan and Ninh, Anh and Zhao, Hui and Liu, Zhenming},
  journal={Production and Operations Management},
  volume={30},
  number={9},
  pages={3231--3252},
  year={2021},
  publisher={SAGE Publications Sage CA: Los Angeles, CA}
}

@article{yadav2015health,
  title={Health product supply chains in developing countries: diagnosis of the root causes of underperformance and an agenda for reform},
  author={Yadav, Prashant},
  journal={Health systems \& reform},
  volume={1},
  number={2},
  pages={142--154},
  year={2015},
  publisher={Taylor \& Francis}
}

@article{zuma2019challenges,
  title={Challenges associated with provision of essential medicines in the Republic of South Africa and other selected African countries},
  author={Zuma, Sibusiso Memory and Modiba, Lebitsi Maud},
  journal={World J Pharm Res},
  volume={8},
  pages={1532--1547},
  year={2019}
}

@article{fredrick2016factors,
  title={Factors influencing frequent stock-outs of essential medicines in public health facilities in Kisii County, Kenya},
  author={Fredrick, Magak Wanga and Muturi, Willy},
  journal={IOSR J Business Manag},
  volume={18},
  number={10},
  pages={63--75},
  year={2016}
}

@article{yenet2023challenges,
  title={Challenges to the availability and affordability of essential medicines in African countries: a scoping review},
  author={Yenet, Aderaw and Nibret, Getinet and Tegegne, Bantayehu Addis},
  journal={ClinicoEconomics and Outcomes Research},
  pages={443--458},
  year={2023},
  publisher={Taylor \& Francis}
}

@misc{WorldPopData,
  author = {{WorldPop}},
  title = {WorldPop Global Project Population Data: Estimated Residential Population per 100x100m Grid Square},
  year = {2021},
  howpublished = {\url{https://www.worldpop.org}},
  publisher = {WorldPop, University of Southampton},
}

@misc{mSupply,
  author = {mSupply},
  year={2024},
  note = {https://www.msupply.org.nz},
}

@article{arkhangelsky2021synthetic,
  title={Synthetic difference-in-differences},
  author={Arkhangelsky, Dmitry and Athey, Susan and Hirshberg, David A and Imbens, Guido W and Wager, Stefan},
  journal={American Economic Review},
  volume={111},
  number={12},
  pages={4088--4118},
  year={2021},
  publisher={American Economic Association 2014 Broadway, Suite 305, Nashville, TN 37203}
}

@misc{clarke2023synthetic,
      title={Synthetic Difference In Differences Estimation}, 
      author={Damian Clarke and Daniel Pailañir and Susan Athey and Guido Imbens},
      year={2023},
      eprint={2301.11859},
      archivePrefix={arXiv},
      primaryClass={econ.EM}
}

@article{abadie2010synthetic,
  title={Synthetic control methods for comparative case studies: Estimating the effect of California’s tobacco control program},
  author={Abadie, Alberto and Diamond, Alexis and Hainmueller, Jens},
  journal={Journal of the American statistical Association},
  volume={105},
  number={490},
  pages={493--505},
  year={2010},
  publisher={Taylor \& Francis}
}

@article{mbonyinshuti2021prediction,
  title={The prediction of essential medicines demand: a machine learning approach using consumption data in Rwanda},
  author={Mbonyinshuti, Francois and Nkurunziza, Joseph and Niyobuhungiro, Japhet and Kayitare, Egide},
  journal={Processes},
  volume={10},
  number={1},
  pages={26},
  year={2021},
  publisher={MDPI}
}

@manual{Kenya2016FamilyPlanning,
  title = {National Guidelines for Quantification, Procurement, and Pipeline Monitoring for Family Planning Commodities in Kenya},
  author = {{Reproductive and Maternal Health Services Unit, Ministry of Health}},
  year = {2016},
  month = {Sep},
  address = {Nairobi, Kenya},
  organization = {Ministry of Health}
}

@manual{USAID2014HealthCommodities,
  title = {Quantification of Health Commodities: A Guide to Forecasting and Supply Planning for Procurement},
  author = {{USAID | DELIVER PROJECT, Task Order 4}},
  year = {2014},
  address = {Arlington, Va},
  publisher = {USAID | DELIVER PROJECT, Task Order 4}
}

@techreport{whosara2018,
    author    = {{World Health Organization}},
    title     = {Service Availability and Readiness Assessment ({SARA}): An Annual Monitoring System for Service Delivery},
    year      = {2018},
    institution = {World Health Organization}
}

@article{bastani2021efficient,
  title={Efficient and targeted COVID-19 border testing via reinforcement learning},
  author={Bastani, Hamsa and Drakopoulos, Kimon and Gupta, Vishal and Vlachogiannis, Ioannis and Hadjichristodoulou, Christos and Lagiou, Pagona and Magiorkinis, Gkikas and Paraskevis, Dimitrios and Tsiodras, Sotirios},
  journal={Nature},
  volume={599},
  number={7883},
  pages={108--113},
  year={2021},
  publisher={Nature Publishing Group UK London}
}

@article{tetteh2009creating,
  title={Creating reliable pharmaceutical distribution networks and supply chains in African countries: Implications for access to medicines},
  author={Tetteh, Ebenezer},
  journal={Research in Social and Administrative Pharmacy},
  volume={5},
  number={3},
  pages={286--297},
  year={2009},
  publisher={Elsevier}
}

@article{linnander2017process,
  title={Process evaluation of knowledge transfer across industries: leveraging Coca-Cola’s supply chain expertise for medicine availability in Tanzania},
  author={Linnander, Erika and Yuan, Christina T and Ahmed, Shirin and Cherlin, Emily and Talbert-Slagle, Kristina and Curry, Leslie A},
  journal={PloS one},
  volume={12},
  number={11},
  pages={e0186832},
  year={2017},
  publisher={Public Library of Science San Francisco, CA USA}
}

@techreport{ICER2019ThresholdRanges,
  author       = {ICER},
  title        = {Perspectives on Cost-Effectiveness Threshold Ranges},
  institution  = {Institute for Clinical and Economic Review},
  year         = {2019},
  type         = {White Paper / Seminar Report},
  url          = {https://icer.org/wp-content/uploads/2023/08/ICER_2019_Perspectives-on-Cost-Effectiveness-Threshold-Ranges.pdf},
}

@misc{HERC_CEA,
  author       = {Health Economics Resource Center (HERC), U.S. Department of Veterans Affairs},
  title        = {Cost‐Effectiveness Analysis},
  howpublished = {\url{https://www.herc.research.va.gov/include/page.asp?id=cost-effectiveness-analysis}}
}

@misc{simplemaps_worldcities_2025,
  title        = {World Cities Database},
  author       = {{SimpleMaps}},
  howpublished = {\url{https://simplemaps.com/data/world-cities}},
  note         = {Last updated May 11, 2025. License: Creative Commons Attribution 4.0.},
  year         = {2025},
  url          = {https://simplemaps.com/data/world-cities}
}

@article{neumann2014updating,
  title={Updating cost-effectiveness—the curious resilience of the $50,000-per-QALY threshold},
  author={Neumann, Peter J and Cohen, Joshua T and Weinstein, Milton C and others},
  journal={N Engl J Med},
  volume={371},
  number={9},
  pages={796--797},
  year={2014}
}

@techreport{projectlastmile2023,
  title        = {2023 Annual Report},
  author       = {{Project Last Mile}},
  year         = 2023,
  institution  = {Project Last Mile},
  url          = {https://www.projectlastmile.com/wp-content/uploads/2024/06/2023-PLM-Annual-Report-240605s.pdf}
}

@article{linnander2018mixed,
  title={A mixed methods evaluation of a multi-country, cross-sectoral knowledge transfer partnership to improve health systems across Africa},
  author={Linnander, Erika and LaMonaca, Katherine and Brault, Marie A and Vyavahare, Medha and Curry, L},
  journal={International Journal of Multiple Research Approaches},
  volume={10},
  number={1},
  pages={136--148},
  year={2018}
}

@article{candes2012exact,
  title={Exact matrix completion via convex optimization},
  author={Candes, Emmanuel and Recht, Benjamin},
  journal={Communications of the ACM},
  volume={55},
  number={6},
  pages={111--119},
  year={2012},
  publisher={ACM New York, NY, USA}
}

@article{caruana1997multitask,
  title={Multitask learning},
  author={Caruana, Rich},
  journal={Machine learning},
  volume={28},
  pages={41--75},
  year={1997},
  publisher={Springer}
}

@article{bertsimas2006robust,
  title={A robust optimization approach to inventory theory},
  author={Bertsimas, Dimitris and Thiele, Aur{\'e}lie},
  journal={Operations research},
  volume={54},
  number={1},
  pages={150--168},
  year={2006},
  publisher={INFORMS}
}

@incollection{bastani2024optimizing,
  title={Optimizing Health Supply Chains in LMICs with Machine Learning: A Case Study in Sierra Leone},
  author={Bastani, Hamsa and Bastani, Osbert and Chung, Tsai-Hsuan},
  booktitle={Responsible and Sustainable Operations: The New Frontier},
  pages={187--202},
  year={2024},
  publisher={Springer}
}

@article{aviv2001effect,
  title={The effect of collaborative forecasting on supply chain performance},
  author={Aviv, Yossi},
  journal={Management science},
  volume={47},
  number={10},
  pages={1326--1343},
  year={2001},
  publisher={INFORMS}
}

@misc{didan2015mod13q1,
  author = {Didan, K.},
  title = {{MOD13Q1 MODIS/Terra Vegetation Indices 16-Day L3 Global 250m SIN Grid V006 [Data set]}},
  year = {2015},
  publisher = {NASA EOSDIS Land Processes Distributed Active Archive Center},
  howpublished = {\url{https://doi.org/10.5067/MODIS/MOD13Q1.006}}
}

@article{xu2025multitask,
  title={Multitask learning and bandits via robust statistics},
  author={Xu, Kan and Bastani, Hamsa},
  journal={Management Science},
  year={2025},
  publisher={INFORMS}
}

@misc{ashenfelter1984using,
  title={Using the longitudinal structure of earnings to estimate the effect of training programs},
  author={Ashenfelter, Orley C and Card, David},
  year={1984},
  publisher={National Bureau of Economic Research Cambridge, Mass., USA}
}

@article{Stevens2016GATHER,
  author    = {Stevens, Gretchen A. and Alkema, Leontine and Black, Robert E. and Boerma, Ties J. and Collins, Gary S. and Ezzati, Majid and others},
  title     = {Guidelines for accurate and transparent health estimates reporting: The GATHER statement},
  journal   = {The Lancet},
  year      = {2016},
  volume    = {388},
  number    = {10062},
  pages     = {e19--e23},
  doi       = {10.1016/S0140-6736(16)30388-9},
}

@article{bastani2021predicting,
  title={Predicting with proxies: Transfer learning in high dimension},
  author={Bastani, Hamsa},
  journal={Management Science},
  volume={67},
  number={5},
  pages={2964--2984},
  year={2021},
  publisher={INFORMS}
}

@misc{HDX_HAPI_Population,
  author    = {{Humanitarian Data Exchange (HDX)}},
  title     = {{HDX HAPI Population Dataset}},
  year      = {2023}, 
  url       = {https://data.humdata.org/dataset/hdx-hapi-population}
}

@techreport{SLCensus2015,
  title        = {2015 Population and Housing Census},
  author       = {{Statistics Sierra Leone}},
  year         = {2015},
  institution  = {Statistics Sierra Leone},
  address      = {Freetown, Sierra Leone},
  url          = {https://www.statistics.sl/index.php/census/census-2015.html},
  type         = {Census Report}
}

@article{ban2019big,
  title={The big data newsvendor: Practical insights from machine learning},
  author={Ban, Gah-Yi and Rudin, Cynthia},
  journal={Operations Research},
  volume={67},
  number={1},
  pages={90--108},
  year={2019},
  publisher={INFORMS}
}

@article{kotary2021end,
  title={End-to-end constrained optimization learning: A survey},
  author={Kotary, James and Fioretto, Ferdinando and Van Hentenryck, Pascal and Wilder, Bryan},
  journal={arXiv preprint arXiv:2103.16378},
  year={2021}
}

@article{shahdecision,
  title={Decision-focused learning without decision-making: Learning locally optimized decision losses},
  author={Shah, Sanket and Wang, Kai and Wilder, Bryan and Perrault, Andrew and Tambe, Milind},
  journal={Advances in Neural Information Processing Systems},
  volume={35},
  pages={1320--1332},
  year={2022}
}

@inproceedings{wilder2019melding,
  title={Melding the data-decisions pipeline: Decision-focused learning for combinatorial optimization},
  author={Wilder, Bryan and Dilkina, Bistra and Tambe, Milind},
  booktitle={Proceedings of the AAAI Conference on Artificial Intelligence},
  volume={33},
  pages={1658--1665},
  year={2019}
}

@article{wang2020automatically,
  title={Automatically learning compact quality-aware surrogates for optimization problems},
  author={Wang, Kai and Wilder, Bryan and Perrault, Andrew and Tambe, Milind},
  journal={Advances in Neural Information Processing Systems},
  volume={33},
  pages={9586--9596},
  year={2020}
}

@article{bertsimas2020predictive,
  title={From predictive to prescriptive analytics},
  author={Bertsimas, Dimitris and Kallus, Nathan},
  journal={Management Science},
  volume={66},
  number={3},
  pages={1025--1044},
  year={2020},
  publisher={INFORMS}
}

@article{elmachtoub2022smart,
  title={Smart “predict, then optimize”},
  author={Elmachtoub, Adam N and Grigas, Paul},
  journal={Management Science},
  volume={68},
  number={1},
  pages={9--26},
  year={2022},
  publisher={INFORMS}
}

@article{kallus2022stochastic,
  title={Stochastic optimization forests},
  author={Kallus, Nathan and Mao, Xiaojie},
  journal={Management Science},
  year={2022},
  publisher={INFORMS}
}

@misc{dhis2,
  title        = {DHIS2 — Digital Health Information System 2},
  howpublished = {\url{https://dhis2.org/}},
  year         = {2025}
}

@misc{GlobalDataLab,
  author       = {Global Data Lab},
  title        = {Global Data Lab — Open data on human development, health, education, wealth, and more},
  year         = {2025},
  howpublished = {\url{https://globaldatalab.org/}}
}

@techreport{WHO,
  author      = "World Health Organization (WHO) Team",
  title       = {Allocation logic and algorithm to support
allocation of vaccines secured through the
COVAX Facility},
  institution = "World Health Organization",
  year        = "2021"
}

@article{agrawal2019differentiable,
  title={Differentiable convex optimization layers},
  author={Agrawal, Akshay and Amos, Brandon and Barratt, Shane and Boyd, Stephen and Diamond, Steven and Kolter, J Zico},
  journal={Advances in neural information processing systems},
  volume={32},
  year={2019}
}

@article{huang2020catalytic,
  title={Catalytic prior distributions with application to generalized linear models},
  author={Huang, Dongming and Stein, Nathan and Rubin, Donald B and Kou, SC},
  journal={Proceedings of the National Academy of Sciences},
  volume={117},
  number={22},
  pages={12004--12010},
  year={2020},
  publisher={National Acad Sciences}
}

@article{yang2021multilocation,
  title={Multilocation newsvendor problem: Centralization and inventory pooling},
  author={Yang, Chaolin and Hu, Zhenyu and Zhou, Sean X},
  journal={Management science},
  volume={67},
  number={1},
  pages={185--200},
  year={2021},
  publisher={INFORMS}
}

@article{braa2007building,
  title={Building collaborative networks in Africa on health information systems and open source software development-experiences from the HISP/BEANISH Network},
  author={Braa, Jorn and Muquinge, Humberto},
  journal={IST Africa},
  volume={3},
  year={2007}
}

@article{rubin1976inference,
  title={Inference and missing data},
  author={Rubin, Donald B},
  journal={Biometrika},
  volume={63},
  number={3},
  pages={581--592},
  year={1976},
  publisher={Oxford University Press}
}

@misc{card1993minimum,
  title={Minimum wages and employment: A case study of the fast food industry in New Jersey and Pennsylvania},
  author={Card, David and Krueger, Alan B},
  year={1993},
  publisher={National Bureau of Economic Research Cambridge, Mass., USA}
}

@misc{angrist1995identification,
  title={Identification and estimation of local average treatment effects},
  author={Angrist, Joshua and Imbens, Guido},
  year={1995},
  publisher={National Bureau of Economic Research Cambridge, Mass., USA}
}

@article{donti2017task,
  title={Task-based end-to-end model learning in stochastic optimization},
  author={Donti, Priya and Amos, Brandon and Kolter, J Zico},
  journal={Advances in neural information processing systems},
  volume={30},
  year={2017}
}

@techreport{ghndr2018sierra,
  title        = {Africa Health Service Delivery in Sierra Leone: Results of 2018 Service Delivery Indicator Survey},
  institution  = {GHNDR and GEDDR Africa, World Bank},
  year         = {2018},
  month        = {June}
}

@misc{nmsa_act2017,
  title        = {National Medical Supplies Agency Act, 2018},
    author = {NMSA}, 
  howpublished = {\url{https://nmsa.gov.sl/index.php/the-nmsa-act/}},
  year         = {2017}
}

@misc{WHO_2024_GlobalHealthEstimates,
  author       = {WHO},
  title        = {Global health estimates: Leading causes of DALYs},
  howpublished = {\url{https://www.who.int/data/gho/data/themes/mortality-and-global-health-estimates/global-health-estimates-leading-causes-of-dalys}},
  year         = {2024},
  institution  = {World Health Organization},
}

@article{droti2019poor,
  title={Poor availability of essential medicines for women and children threatens progress towards Sustainable Development Goal 3 in Africa},
  author={Droti, Benson and O’Neill, Kathryn Patricia and Mathai, Matthews and Dovlo, Delanyo Yao Tsidi and Robertson, Jane},
  journal={BMJ Global Health},
  volume={4},
  number={Suppl 9},
  pages={e001306},
  year={2019},
  publisher={BMJ Specialist Journals}
}

@article{yang2012walking,
  title={Walking distance by trip purpose and population subgroups},
  author={Yang, Yong and Diez-Roux, Ana V},
  journal={American journal of preventive medicine},
  volume={43},
  number={1},
  pages={11--19},
  year={2012},
  publisher={Elsevier}
}

@inproceedings{ludwig2024unreasonable,
  title={The unreasonable effectiveness of algorithms},
  author={Ludwig, Jens and Mullainathan, Sendhil and Rambachan, Ashesh},
  booktitle={AEA Papers and Proceedings},
  volume={114},
  pages={623--627},
  year={2024},
  organization={American Economic Association 2014 Broadway, Suite 305, Nashville, TN 37203}
}

@misc{WorldBank2025,
  author       = {{World Bank}},
  title        = {Crude birth rate (per 1,000 people) - Sierra Leone},
  year         = {2025},
  howpublished = {\url{https://data.worldbank.org/indicator/SP.DYN.CBRT.IN?locations=SL}},
  note         = {Accessed: 2025-08-21}
}

@misc{UNICEF_2025_SierraLeoneData,
  author       = {UNICEF},
  title        = {Sierra Leone: Demographics, Health \& Infant Mortality — Country Profile},
  howpublished = {\url{https://data.unicef.org/country/sle/}},
  year         = {2025},
  institution  = {United Nations Children’s Fund},
}

\bibliographystyle{Science}

\section*{End Notes}

\paragraph*{Acknowledgments}
We are grateful for the close partnership offered by the Sierra Leone Ministry of Health and Sanitation and the National Medical Supplies Agency; this partnership stemmed from an earlier collaboration with Macro-Eyes and its team members (Suvrit Sra, Vahid Rostami, Ashley Schmidt, Musa Komeh, Lydia Bernard-Jones, and Rene Ishiwe). We acknowledge invaluable research assistance from our team of RAs (Allan Zhang, Cheng-Ying Wu, Norris Chen, and Hingis Chang). This draft benefited from helpful feedback from Prashant Yadav, Gerard Cachon, Gad Allon, and participants at the Machine Learning for Health, Symposium on Artificial Intelligence in Learning Health Systems (SAIL), INFORMS Annual Conference, Marketplace Innovation Workshop, MSOM Healthcare SIG, Purdue Operations Conference, and Workshop on AI \& Analytics for Social Good.

\paragraph*{Funding:} This research was supported by generous funding from the Wharton AI \& Analytics Initiative, Wharton Global Initiatives, and the Wharton Mack Institute for Innovation Management. 

\paragraph*{Author contributions:}
HB, OB, and AC contributed to the design and methodology of the research and to the writing of the manuscript. AC analyzed the results, supervised by HB and OB. AC, PB, JA, LS, and FS contributed to the implementation of the research. 

\paragraph*{Competing interests:}

There are no competing interests to declare.

\paragraph*{Ethics \& Inclusion Statement:} This research was conducted in close partnership with the Sierra Leone Ministry of Health and Sanitation (MoHS) and the National Medical Supplies Agency (NMSA), subject to a data-sharing Non-Disclosure Agreement (NDA) and a formal Memorandum of Understanding (MOU) with the Government of Sierra Leone. The MoHS and NMSA oversaw the system’s deployment and integration into national workflows. We collaborated closely with policymakers and frontline personnel, conducted training sessions, and aligned our approach with existing government systems and priorities. Roles and responsibilities were clearly defined in advance, and system ownership was transferred to the government after nationwide deployment.

\paragraph*{Materials \& Correspondence.}  Correspondence should be addressed to hamsab@wharton.upenn.edu and obatani@seas.upenn.edu

\newpage
\setcounter{section}{0}
\setcounter{page}{1}

\setcounter{figure}{0}
\renewcommand{\thefigure}{SI~\arabic{figure}}

\setcounter{table}{0}
\renewcommand{\thetable}{SI~\arabic{table}}

\begin{center}
\section*{Supplementary Information}
\end{center}

\S\ref{sup:one} describes the data and methods supporting the design of our allocation system; \S\ref{sup:evaldeploy} describes the deployment of our system and the empirical evaluation of its effectiveness.

\section{Data and Methods}
\label{sup:one}

First, we outline our data sources (\S\ref{sup:datasets}) and detail how we processed the raw data for training our machine learning model and our evaluation (\S\ref{sup:dataprocessing}). Second, we review the allocation mechanism used in Sierra Leone prior to adopting our approach (\S\ref{sup:ExistingAllocation}). Then, we formalize resource allocation as an optimization problem (\S\ref{sup:optimization}). Finally, we present our end-to-end machine learning pipeline for demand estimation and decision-aware allocation (\S\ref{sup:mlframework}), which spans: 
\begin{itemize}
\item Multi-Task Learning (\S\ref{sup:multitasklearning}): We exploited cross-facility and cross-product patterns to improve predictive accuracy with limited data.
\item Catalytic Priors (\S\ref{sup:catalyticpriors}): We constructed catalytic priors from population estimates to address data quality issues such as censoring and missing values. 
\item Decision-Aware Learning (\S\ref{sup:dawarelearning}): We proposed a novel method for re-weighting training data to align the prediction objective with the downstream optimization objective, thus shifting the focus from prediction accuracy to improving public health. 
\end{itemize}

\subsection{Datasets}
\label{sup:datasets}

\paragraph{List of public health facilities.}

We obtained information on the ID, latitude, longitude, and type of public health facilities in Sierra Leone~\cite{SierraLeoneMoHS_CIESIN_2023}; this data was cross-verified with frontline staff at the National Medical Supplies Agency (NMSA).

\paragraph{Consumption data.} We constructed outcomes and features for demand prediction based on data extracted from the District Health Information Software 2 (DHIS2) used by Sierra Leone Ministry of Health and Sanitation (MoHS) to collect and manage health data. The country transitioned from paper-based to electronic reporting of health data in public health facilities in 2019~\cite{martin2020implementing}. We extracted monthly facility-level data on consumption, opening balance, closing balance, and stockouts of 62 medicines and medical supplies across all the public health facilities from January 2020 to November 2023. 36 of these products were allocated via our tool (see Table~\ref{tab:medicine_consumption_treat1}) and the remaining 26 were allocated via existing mechanisms (see Table~\ref{tab:medicine_consumption_treat0}); the latter was either because the allocation for these products was controlled by an external third party,
or because there was too little historical data to use our system. Data from products allocated via existing mechanisms was used only to support our multi-task prediction model, and for the robustness check of our main analysis using product-level controls (``Alt. Control'').

\paragraph{Supply data.}
We collected expiry dates\footnote{In line with the NMSA's existing practice, we prioritized allocating near-expiry products to larger districts, where higher consumption minimizes waste by ensuring supplies are used up prior to expiration.} and available supply of all medicines and medical supplies per quarter from mSupply~\cite{mSupply}, a pharmaceutical logistics and warehouse management system used in Sierra Leone (and over 40 other countries). Prior to each allocation quarter, local staff are required to conduct central stock counts and record the information in mSupply. In addition, we used invoice records from the mSupply system to determine the stock received by each public health facility to evaluate compliance for district-level allocations.

\paragraph{Catchment population.} 
Estimating granular population is particularly challenging in developing countries. To address this challenge, we leveraged multiple publicly available datasets to estimate each health facility's catchment population (i.e., the number of people each health facility is expected to serve): the WorldPop Global Project Population Data~\cite{WorldPopData}, the global friction surface dataset (\verb|Oxford/MAP/friction_surface_2019|) from Google Earth~\cite{weiss2020global}, and satellite imagery~\cite{didan2015mod13q1} also accessed through Google Earth. We used this data to create population estimates for our catalytic priors (see \S\ref{sup:catalyticpriors}).

\subsection{Data Processing} 
\label{sup:dataprocessing}

\paragraph{Training data for demand prediction model.} The data for training our demand prediction model is derived from the historical consumption data extracted from DHIS2. We used this time series data to construct both the demand $\xi_{t,n}^*$ in the current period $t$ for facility $n$ that we are trying to predict as well as the features $x_{t,n}$ for prediction. For example, we constructed facility-specific features, including: consumption, product, facility ID, facility type, latitude and longitude of the facility's geo-location, district, average consumption of the product for the facility in the past $\{1,2,3,4,5,6\}$ months, standard deviation of the consumption in the past 3 and 6 months, total sample size for the facility-product pair, year, month, average consumption of the product across facilities in the past $\{1,2,3,4,5,6,10\}$ months. These features were determined based on domain knowledge and feature engineering.

This data can sometimes be unreliable due to random or inconsistent data entry at individual facilities, requiring careful preprocessing. First, we excluded any observations where the inflow and outflow are inconsistent---i.e.,
\begin{align*}
&\text{Closing Balance} \\
&\neq\text{Opening Balance} + \text{Quantity Received} - \text{Quantity Dispensed} + \text{Adjustment/Loss}
\end{align*}
Second, we removed observations where all recorded quantities were zero, since this likely indicates the use of a default value. Third, we excluded extreme outliers---specifically, the top and bottom 5\% of values, to mitigate the impact of potential data entry errors.

One major challenge is that we only observe consumption, which does not equal demand when there is a stockout; this issue is called \emph{demand censoring}~\cite{besbes2013implications}. To ensure we are predicting actual demand, we used a standard strategy where we dropped censored observations~\cite{greene2000econometric}. In particular, when constructing $(x_{t,n},\xi_{t,n}^*)$ pairs for training, we only included observations where no stockout occurred; then, consumption is equal to the demand, so we can take it to be $\xi_{t,n}^*$. One limitation of this strategy is that it introduces covariate shift, since there may be systematic differences between time periods/facilities where stockouts occur and those where they do not occur. Covariate shift has the potential to degrade performance of the model compared to what is expected based on test set evaluation. Thus, we used unbiased population-based models that do not suffer from censoring as a catalytic prior when training our predictive models to mitigate this covariate shift (see \S\ref{sup:catalyticpriors}).

\paragraph{Panel data for evaluation.}

Our econometric evaluation uses the same historical DHIS2 data as above. In this case, the main preprocessing we performed is to remove unreliable data---i.e., excluding any observations where the inflow and outflow are inconsistent and removing observations where all recorded quantities were zero, as described above. Importantly, our evaluation is performed using consumption outcomes, which are not affected by demand censoring (we also perform a number of checks to ensure robustness to data missingness; see \S\ref{sup:robustnesschecks}). As we show in \S\ref{sup:consumption}, increasing consumption is mathematically equivalent to reducing unmet demand.

We used data from 2022 Q3 to 2023 Q3. We started in 2022 Q3 since this was the quarter when the NMSA began using a standardized, data-driven Excel allocation tool by Crown Agents; prior to this quarter, allocation procedures were inconsistent, which can prevent reliable counterfactual estimation. We constructed a balanced panel dataset with 1,058 facilities for evaluation.

The scale of quarterly consumption at each facility depends on the type and size of the facility's catchment population. To account for these differences, we normalized each product's consumption at the facility level by subtracting the mean consumption of that product across all facilities and dividing by the standard deviation:
\begin{align*}
\text{NormalizedConsumption}_{n,m} &= \frac{\text{Consumption}_{n,m} - \text{MeanConsumption}_m}{\text{StdConsumption}_m},
\end{align*}
where $n$ represents the facility, $m$ represents the product, $\text{MeanConsumption}_m$ is the average consumption of product $m$ across facilities, and $\text{StdConsumption}_m$ is the standard deviation of product $m$ across facilities. For each facility in each quarter, we then calculated the average normalized consumption that is available at that facility as:
\begin{align*}
\text{FacilityAverageNormalizedConsumption}_n &= \frac{\sum_{m \in \text{AvailableProducts}_n} \text{NormalizedConsumption}_{n,m}}{\text{AvailableProducts}_n}.
\end{align*}
Typically, this sum is only over the products that we allocated via our tool (except in our staggered SynthDiD approach, where we compared average normalized consumption of products we allocated vs. products we did not allocate). This approach ensures that the facility-level average consumption reflects the relative performance across products while accounting for the variability in the consumption level of different products. We show the time series trends of average normalized consumption for treatment vs. control groups in Fig.~\ref{fig:raw} in the main paper; these raw data trends qualitatively support our main findings in the paper.

\subsection{Existing Allocation Approach}
\label{sup:ExistingAllocation}

Each quarter, the National Medical Supplies Agency (NMSA) of Sierra Leone allocates approximately 70-100 free healthcare products specifically for women and children under five years old, with supply primarily dependent on international donations. Our system focuses exclusively on allocating these products. Notably, the National Medical Supplies Agency (NMSA) Act of 2017 establishes the NMSA as the sole entity responsible for the procurement and distribution of these products to all public health institutions~\cite{nmsa_act2017}. Public facilities are the only source where patients can obtain these specific medications for free. While private providers may exist, they are not part of this supply chain, and a majority of patients are unlikely to purchase expensive drugs that are available for free. Thus, efficient allocation is critical for equitable access.

The distribution of central stock to facilities occurs quarterly based on a centralized two-stage push system~\cite{terwiesch2006matching}, where supplies move from the central government to districts, and then to local health facilities. We focused on a subset of 36 products that are regularly distributed---chosen in collaboration with NMSA officials prior to our deployment---ensuring sufficient historical data for model training.

Until the deployment of our tool in 2023 Q2, the process of computing the allocation of central stock to each health facility relied primarily on a complex Excel tool. An important aspect was organizing health facilities into three administrative categories: District Medical Stores (DMS), District Hospitals (DH), and Western Area Hospitals (WAH). The DMS includes four facility types: Community Health Centers (CHC), Community Health Posts (CHP), Maternal and Child Health Posts (MCHP), and Clinics. Prior to each allocation cycle, all public health facilities submit requests based on their recent three-month rolling average of consumption. Workers can provide this information based on DHIS2, mSupply, or their professional judgment. 

Upon receiving all facility requests, the NMSA implements a structured allocation process. First, the NMSA determines the distribution proportions among the three primary healthcare facility categories, typically 70\% of total stock to DMS across 16 districts, 15\% to DH, and 15\% to WAH. Following this initial distribution, each district receives a specific allocation based on multiple criteria, including district population, poverty levels, product types, and submitted requests. For example, if a DMS in a particular district is allocated 10\% of the DMS share, it receives a quantity calculated as
\begin{align*}
\text{total available central stock} \times 0.7 \times 0.1.
\end{align*}
In cases where central stock remains after the initial distribution and requests remain partially fulfilled, the NMSA makes further allocations based on the unfulfilled requests.

Once the allocations have been finalized, the distribution process follows a two-tier delivery system. The NMSA executes the ``first mile'' delivery to all districts, after which each district manages the ``last mile'' distribution to individual health facilities under the NMSA's guidance and supervision.

Our system was designed to address several limitations of this Excel-based approach. Most importantly, the existing approach predicted demand based on the average demand over the previous three months; while this strategy is standard, it fails to capture shifts in demand due to seasonal fluctuations or demand spikes due to natural disasters or disease outbreaks. Our system achieves better demand forecasting through a combination of multi-task learning, catalytic priors, and decision-aware learning; in addition, it also significantly automates and improves the data processing pipeline. The Excel tool relied on manually collecting requests at the district level; however, due to tight timelines for making allocation decisions, the government often used outdated records instead. Our system also performs systematic data validation to identify potentially erroneous entries such as unexpectedly large numbers of zeros or inconsistencies in stock balance calculations. Finally, our system incorporates stock expiry dates into allocation decisions, which prioritizes sending near-expiring products to larger facilities with more consistent demand to minimize waste.

\subsection{Supply Chain Optimization Algorithm}

We designed an optimization algorithm to allocate limited medical resources to health facilities to minimize total unmet demand, defined as the total number of units of requested product that goes unfulfilled. The unmet demand objective is aligned with the goals of NMSA policymakers since it reflects the quality of care provision and directly affects patient outcomes, particularly in regions with limited access to alternative care options.

Our system optimizes the allocation of each product separately.
There is a fixed total available central stock $b\in\mathbb{R}$ (which we refer to as the \emph{budget}) to be distributed across $N\in\mathbb{N}$ facilities. Each facility $n\in[N]$ has an estimated demand $\Xi_n$, where $\Xi\in\mathbb{R}^N$ is a real-valued random vector with an estimated distribution $\mathbb{P}_{\Xi}$. We denote the allocation decision and facility stock on hand as $a\in\mathbb{R}^N$ and $s\in\mathbb{R}^N$, respectively, where $a_n$ is the allocation intended for facility $n$ and $s_n$ is the stock on hand at each facility $n$. Demand $\Xi_n$, allocation $a_n$, and facility stock $s_n$ are all measured in units of product. Then, the expected unmet demand measures the amount of unmet demand on average across facilities and over $\Xi$:
\begin{align}
\label{eq:obj}
\mathbb{E}_{\Xi}\left[\sum_{n\in[N]}\max\{\Xi_n-a_n-s_n,0\}\right].
\end{align}
If the budget $b$ is very large, we can choose all $a_n$ sufficiently high to ensure that our objective is minimized at approximately $0$. If the budget $b$ is very low compared to the total excess demand $\sum_{n\in [N]} (\Xi_n - s_n)^+$, then most facilities will suffer stockouts, and we cannot do much better than 
$(\sum_{n\in [N]} (\Xi_n - s_n)^+) - b.$ However, we found that the budget is often on the order of the total excess demand (i.e., $b \approx \sum_{n\in [N]} (\Xi_n - s_n)^+$), potentially because the budget is adjusted over time to meet demand. In this case, many facilities are both over- and under-stocked; therefore, minimizing the objective requires setting the allocations $a_n$ to be as close to the excess demand $(\Xi_n - s_n)^+$ as possible.

We focused on allocating over a single period instead of multi-period allocation. This is because we found that forecasting demand beyond one quarter is far too noisy to be of value.

\paragraph{Optimization strategy.}
\label{sup:optimization}

When $\Xi$ is constant, the optimal policy can be straightforwardly expressed as a linear program. To account for the uncertainty in $\Xi$, we use \emph{sample average approximation} (SAA), which takes $K$ demand samples from the estimated demand distribution $\xi^{(k)}\sim\mathbb{P}_{\Xi}$ (for $k\in[K]$), and then optimizes the objective on average across these samples. The resulting optimization problem is
\begin{align}
\label{eqn:lp}
a^* = &\operatorname*{arg\,min}_{a\in\mathbb{R}^N}\frac{1}{K}\sum_{k=1}^K\sum_{n=1}^N c_n^{(k)} \\
&\text{subj. to}\quad
  c_n^{(k)}\ge \xi_n^{(k)}-a_n-s_n,
  \quad \forall k\in[K],\ \forall n\in[N], \nonumber\\
&\phantom{\text{subj. to}\quad}
  c_n^{(k)}\ge 0,
  \quad \forall k\in[K],\ \forall n\in[N], \nonumber\\
&\phantom{\text{subj. to}\quad}
  a_n\ge 0,
  \quad \forall n\in[N], \nonumber\\
&\phantom{\text{subj. to}\quad}
  \sum_{n=1}^N a_n\le b. \nonumber
\end{align}

where vector inequalities are element-wise, $c_n^{(k)}$ denotes the unmet demand for facility $n$ in sample $k$, and $b$ is the budget. The first two constraints ensure  $c_n^{(k)}=\max\{\xi_n^{(k)}-a_n-s_n,0\}$, with one of these constraints necessarily binding to minimize the objective. The last constraint ensures the total allocation does not exceed the budget.

\paragraph{Warehouse matching.}
\label{sup:warehouse} On the basis of the final allocation generated by our system, we determined the specific inventory in warehouses that should be shipped to each facility. In line with Sierra Leone Ministry of Health's Free Healthcare Initiative policy, we preferentially allocated faster-expiring stock to higher-volume facilities in which it was more likely to be used prior to expiration. In particular, for each product, we first ranked the stock based on time to expiration. Then, we iterated through facilities on the basis of a ranking provided by the NMSA (typically, facilities with larger catchment populations are ranked higher). For each facility, we allocated stock with the earliest expiry date across all warehouses, continuing until the facility's allocation was fully met.

\subsection{Machine Learning Framework} 
\label{sup:mlframework}

We present our machine learning framework for demand forecasting. Our system trains a random forest using multi-task learning and catalytic priors; to do so, it uses a novel decision-aware learning approach to better align predictions with the downstream optimization loss. 

\subsubsection{Multi-Task Learning}
\label{sup:multitasklearning}
A common approach for demand estimation in supply chain management and global health is to fit a distribution to historical consumption patterns, e.g., one could estimate the mean and variance of each facility-product pair separately on the basis of its historical data. Except where otherwise noted, we elide the product from our notation; our algorithm applies the procedure we describe to train one model per product category (specifically, one for medicines and one for medical supplies) and uses it to generate predictions for each product. Mathematically, it can be viewed as solving the following maximum likelihood problem:
\begin{align*}
\tilde{\ell}(\mu,\sigma)=-\sum_{t=1}^T\sum_{n=1}^N\log\mathcal{N}(\xi_{t,n}^*;\mu_n,\sigma_n^2),
\end{align*}
where $\mathcal{N}(x;\mu_n,\sigma_n^2)$ is the Gaussian probability density function in $x$ with mean $\mu_n$ and standard deviation $\sigma_n$ and $\tilde{\ell}(\mu,\sigma)$ is the negative log-likelihood. This objective $\tilde{\ell}$ decomposes across facilities $n\in[N]$, and the solution for a given $n$ is the empirical mean and variance. However, this strategy cannot learn dynamic patterns such as seasonal effects or demand that is elevated for a period of time (e.g., due to an outbreak). Time series models such as ARIMA are also infeasible, owning to the limited number of observations we have for each facility-product pair---instead, we must leverage cross-facility and cross-product correlations. 

Multi-task learning allows us to train a single model on multiple interrelated tasks (i.e., the different facility-product pairs). By aggregating data and transferring knowledge across related tasks, multi-task learning increases the effective sample size for each task~\cite{caruana1997multitask}. Intuitively, demand may exhibit similar patterns across different facilities (e.g., seasonal trends, regional trends), and training a single model enables trends observed in a large fraction of facilities to be extended to other facilities for which data may be more scarce. Consider the following hypothetical example: suppose that the model observes a rise in demand for antimalarial medicines during the rainy season at several data-rich rural clinics; then, it can apply this pattern to other rural clinics for which there is insufficient data to identify this pattern. Thus, patterns in one facility directly inform predictions for other facilities, increasing the overall prediction accuracy with limited data.

In more detail, we first categorized all products into two types (medicines, or medical supplies and equipment, as shown in Tables~\ref{tab:medicine_consumption_treat1} \& \ref{tab:medicine_consumption_treat0}); we trained separate predictive models for these two categories. We associated each demand observation $\xi_{t,n}^*$ with a covariate vector $x_{t,n}\in\mathbb{R}^d$, which consists of features constructed from (1) the facility $n$, (2) the time step $t$, and (3) the historical demand over the $k$ previous steps $\xi_{t-k,n}^*,\xi_{t-k+1,n}^*,...,\xi_{t-1,n}^*$ (it also includes features of the product being allocated), based on domain knowledge from consultations with local medical experts and policy makers along with extensive feature engineering. Specifically, our prediction model uses the following covariates: lagged consumption, product, facility ID, facility type, latitude and longitude of the facility's geo-location, district, average consumption of the product for the facility in the past $\{1,2,3,4,5,6\}$ months, standard deviation of the consumption in the past 3 and 6 months, total sample size for the facility-product pair, year, month, average consumption of the product across facilities in the past $\{1,2,3,4,5,6,10\}$ months.

Then, for each category (medicines or medical supply/equipment), we trained a single model on the resulting dataset $\{(x_{t,n},\xi_{t,n}^*)\}_{t\in[T],n\in[N]}$. In particular, we aimed to train predictors $\mu_{\theta}(x)$ and $\sigma_{\theta}(x)$ for the mean and standard deviation, respectively, where $\theta\in\Theta$ are the parameters, using the following objective:
\begin{align}
\label{eqn:nll}
\tilde{\ell}(\theta)=-\sum_{t=1}^T\sum_{n=1}^N\log\mathcal{N}(\xi_{t,n}^*;\mu_\theta(x_{t,n}),\sigma_\theta(x_{t,n})^2).
\end{align}
In practice, we found that the following strategy works well. First, we used a random forest to fit the mean $\mu_{\theta}(x_{t,n})$, assuming the variance is constant. Then, we fit $\sigma_{\theta}(x_{t,n})$ only on the basis of historical data for the facility $n$ and the current product in question.

To evaluate our multi-task learning strategy, we compared to two techniques that do not use multi-task learning. First, we considered the 3-month rolling average, which is the prediction strategy used by the existing Excel tool
and more broadly is a common strategy for demand forecasting in LMICs~\cite{USAID2014HealthCommodities, Kenya2016FamilyPlanning}. Second, we compared to a standard distribution modeling approach, where we fit a demand distribution for each facility-product pair based only on historical data from that facility-product pair. We examined the fit of 46 well-known candidate distributions (e.g., normal, beta, gamma, exponential) and chose the Nakagami distribution as the one that best minimized the sum of squared errors between the fitted probability density function (PDF) and a histogram of the historical data.

We also compared to two other standard model families, LASSO regression and neural networks (NN), trained using the same multi-task learning pipeline as our random forest.
In this comparison, we focused on evaluating the mean squared error (MSE) of $\mu_{\theta}$, since we fit $\sigma_{\theta}$ separately. Results are shown in Fig.~\ref{fig:PredErrorCP}; as can be seen, random forests consistently achieved lower MSE across representative products on a held-out test set.

\begin{figure}[t]
\begin{center}
\includegraphics[width=0.95\linewidth]{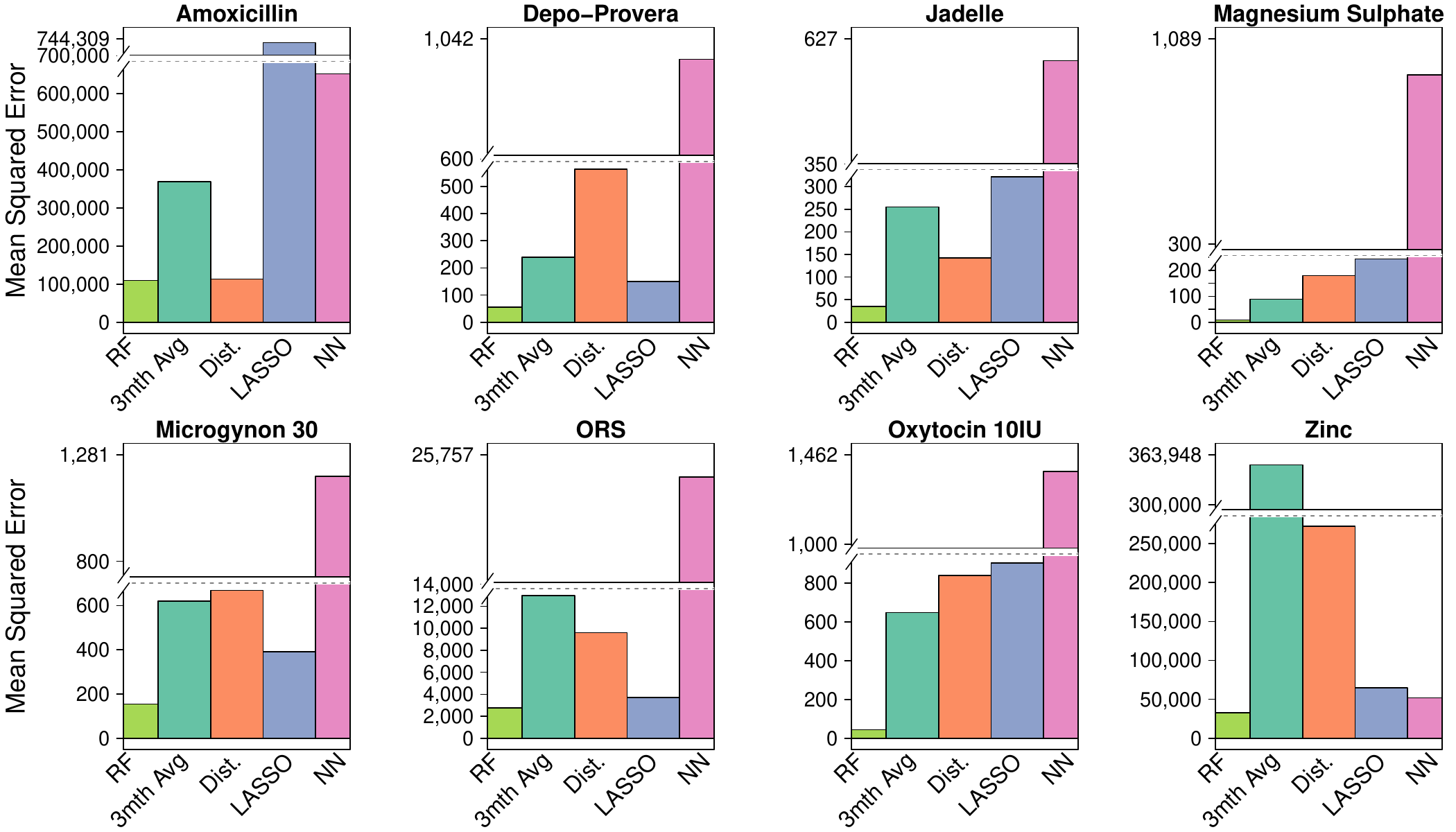}
\caption{\textbf{Prediction error comparison between different model families.} Random forests (RF) yield the smallest prediction error for representative essential medicines chosen by the NMSA, compared to using a rolling 3-month average (3mth Avg), distribution modeling (Dist.) described in \S\ref{sup:multitasklearning}, LASSO regression (LASSO), and neural networks (NN).}
\label{fig:PredErrorCP}
\end{center}
\end{figure}

\subsubsection{Catalytic Priors} 
\label{sup:catalyticpriors}

We implemented \emph{catalytic priors}~\cite{huang2020catalytic} as a safeguard to improve our model's robustness to inequities in data quality (arising from missing data or censoring). Oftentimes, low quality data come from poorer districts---thus, a model trained only on the available data may be biased in poorer districts, which may create unintentional disparities in predictive performance and downstream resource allocation. 

A natural strategy for mitigating such bias is to incorporate auxiliary data sources that are less likely to suffer from bias. For example, in public health, a standard approach is population-based resource allocation (PBRA), which uses population estimates to guide proportional resource allocation needs~\cite{Stevens2016GATHER}. Although population-based prediction is noisier (i.e., higher variance) because it cannot capture time-dependent patterns (e.g., seasonality of demand), it is less biased since the observations do not suffer from nonrandom missingness or censoring. Catalytic priors~\cite{huang2020catalytic} allow us to suitably trade off bias and variance by regularizing our random forest with the simpler but less biased population-based model, which predicts demand only on the basis of an estimate of the at-risk population in the catchment of a health facility.

For a given product, denote the prediction of the population-based model by $\xi_{t,n}^0=\mu_C(p_n)=r\cdot C\cdot p_n$, where $p_n$ is our estimated population in the catchment of facility $n$; $r \in \mathbb{R}$ is an estimated multiplier derived from census data to account for the at-risk population (defined by the percentage of women and children in a given area); and $C\in\mathbb{R}$ is a single \textit{product-specific} parameter estimated from our historical data (i.e., the average quarterly demand per unit of at-risk population). The catchment population $p_n$ served by each health facility is not readily available, so we estimated it as follows:
\begin{enumerate}
\item{First, we collected data on geographic coordinates of health facilities in Sierra Leone from several sources, including Google Maps and Geo-Referenced Infrastructure and Demographic Data for Development (GRID3)~\cite{SierraLeoneMoHS_CIESIN_2023}.}
\item{Next, we used Google Earth engine’s satellite imagery datasets to compute the normalized difference vegetation index (NDVI) on a 10km$\times$10km patch around each facility at monthly resolution between January 2022 and December 2022. NDVI serves as a proxy for vegetation density, which can indicate human activity.\footnote{NDVI is derived from satellite images using the formula: 
$\text{NDVI}=(\text{NIR} - \text{red}) / (\text{NIR} + \text{red})$. Since vegetation reflects light in the near-infrared (NIR) spectrum and absorbs light in the red spectrum, areas with higher photosynthesis activity exhibit larger NDVI values.
}}
\item{Then we used ``friction surface'' data from Google Earth~\cite{weiss2020global} to obtain the travel time between every facility and pixel of the area with potential human activity. This allows us to define the catchment area on the basis of minimal travel time.}
\item{Finally, we estimated $p_n$ for health facility $n$ by using data from WorldPop~\cite{WorldPopData}, which provides population count  estimates for each 100m$\times$100m grid cell.}
\end{enumerate} 
We estimated $r$, the proportion of women and children, using 2015 Sierra Leone Census data~\cite{SLCensus2015} at the chiefdom level.\footnote{Chiefdoms are a more granular administrative unit than districts; there are 190 chiefdoms in Sierra Leone. We further verified these estimates using data from the United Nations Office for the Coordination of Humanitarian Affairs (OCHA)~\cite{HDX_HAPI_Population}}

Then, we followed \cite{huang2020catalytic}, generating synthetic observations from $\mu_C$ to act as a Bayesian prior for our random forest. In particular, we used $\mu_C$ to construct a single synthetic example $(x_{t,n},\mu_C(p_n))$ for each facility $n\in[N]$ and time period $t\in[T]$. We then trained $\mu_{\hat\theta}$ on a weighted combination of the original dataset and this synthetic dataset. Intuitively, the synthetic dataset regularizes $\mu_{\hat\theta}$ towards a stable estimate $\mu_C$ in data-poor regions of the covariate space.

\subsubsection{Decision-Aware Learning}
\label{sup:dawarelearning}

The next step is to incorporate our predictions into the optimization model to generate allocation decisions. However, the model is usually trained to forecast demand using a standard objective such as mean-squared error (MSE), which focuses on minimizing prediction error and ignores the decision error in the downstream optimization problem, making it \emph{decision-blind}. This can result in poor performance since it may not focus the capacity of the machine learning model on predictions that are actually relevant to making decisions \cite{bertsimas2020predictive,elmachtoub2022smart}.

Recent work has proposed algorithms that incorporate the downstream optimization objective as a \emph{decision loss} in the training algorithm~\cite{kotary2021end, bertsimas2020predictive}. We found that existing decision-aware learning algorithms were either computationally intractable at our scale or incompatible with the rest of our prediction and optimization pipeline. Thus, we developed a novel and light-weight decision-aware learning approach; it relies only on re-weighting observations in model training, which can be easily integrated with existing data pipelines.

In our setting, the decision loss is
\begin{align*}
\ell(\hat\mu;\mu^*)=L(a^*(\hat\mu),\mu^*),
\end{align*}
where
\begin{align*}
L(a;\mu)=\sum_{n=1}^N\mathbb{E}_{\Xi_n\sim\mathcal{N}(\mu_n,\sigma^2)}\left[\max\{\Xi_n-a_n-s_n,0\}\right],
\end{align*}
is the unmet demand for allocation $a\in\mathbb{R}^N$ assuming the true demand for facility $n\in[N]$ is $\mathcal{N}(\mu_n,\sigma^2)$ (recall that for training the random forest, we have assumed that the standard deviation is a fixed value $\sigma$), and where
\begin{align*}
a^*(\mu) 
&= \arg\min_{a \in \mathbb{R}^N} L(a;\mu) \\
\text{s.t.}\quad 
& a_n \ge 0 \quad (\forall n \in [N]),\quad \sum_{n=1}^{N} a_n \le b .
\end{align*}

is the optimal allocation assuming the demand distributions for $n\in[N]$ is $\mathcal{N}(\mu_n,\sigma^2)$. In other words, the decision loss is the expected unmet demand incurred when using the predictions $\hat\mu$ in our optimization problem. To address objective mismatch, we could train $\mu_{\theta}$ to directly minimize the decision loss:
\begin{align*}
\hat\theta=\operatorname*{\arg\min}_{\theta\in\Theta}\sum_{t=1}^T\sum_{n=1}^N\ell(\mu_\theta(x_{t,n});\mu_{t,n}^*).
\end{align*}
Algorithms for doing so have been proposed in the setting of linear regression~\cite{elmachtoub2022smart}, and in the more general setting of differentiable model families by taking gradients through the optimization problem~\cite{wilder2019melding, wang2020automatically}. However, existing techniques are often limited to specific prediction setups or become computationally intractable for large-scale problems~\cite{kallus2022stochastic, wilder2019melding, wang2020automatically}. 

Our strategy is to Taylor expand the optimal decision loss, which we will show can be interpreted as up-weighting data more relevant to the downstream optimization problem. This approach can also be easily integrated with existing pipelines and is flexible enough to handle a broad range of model families. In particular, we approximated the decision loss by Taylor expanding it around $\hat\mu-\mu_0$ (where $\mu_0$ is the current prediction and $\hat\mu=f_\theta(x)$), yielding
\begin{align}
\label{eqn:taylor}
L(a^*(\hat{\mu});\mu^*)
\approx
L(a^*(\mu_0);\mu^*)
+
\nabla_aL(a^*(\mu_0);\mu^*)^\top\nabla_\mu a^*(\mu_0)^\top(\hat{\mu}-\mu_0).
\end{align}
Since the first term is a constant, we can ignore it; in particular, we have
\begin{align*}
\ell(\hat\mu;\mu^*)
&\approx\nabla_aL(a^*(\mu_0);\mu^*)^\top\nabla_\mu a^*(\mu_0)^\top(\hat\mu-\mu_0)+\text{const} \\
&=\nabla_aL(a^*(\mu_0);\mu^*)^\top\nabla_\mu a^*(\mu_0)^\top(\hat\mu-\mu^*)+\text{const}.
\end{align*}
Here, we have replaced $\mu_0$ with $\mu^*$; since both of these are constants, it does not affect the optimal solution. With this replacement, we can upper bound the term by its absolute value to avoid ``overshooting'' $\xi^*$, resulting in a weighted absolute error loss:
\begin{align*}
\ell(\hat\mu;\mu^*)
&=\sum_{t=1}^T\sum_{n=1}^Nw_{t,n}(\hat\mu_{t,n}-\mu^*_{t,n})+\text{const} \\
&\le\sum_{t=1}^T\sum_{n=1}^N|w_{t,n}|\cdot|\hat\mu_{t,n}-\mu^*_{t,n}|+\text{const},
\end{align*}
where
\begin{align*}
w_{t,n}=\left(\nabla_{\mu}a^*(\mu_0)^\top\nabla_aL(a^*(\mu_0);\mu^*)\right)_{t,n}
\end{align*}
Our algorithm uses this upper bound as the loss function, which works with any standard machine learning algorithm that can take weighted examples. In particular, we trained our random forest to minimize this loss on the training data:
\begin{align*}
\hat\theta=\operatorname*{\arg\min}_{\theta}\sum_{t=1}^T\sum_{n=1}^N|w_{t,n}|\cdot|\mu_{\theta}(x_{t,n})-\xi_{t,n}^*|.
\end{align*}
Note that we do not observe the true demand $\mu_{t,n}^*$, so we use $\xi_{t,n}^*$ as an estimate. Finally, as a heuristic, we replaced the absolute error with the squared error, which is more computationally efficient.

An important insight here is that the weight can be interpreted as re-weighting training examples. The $w_n$ can be decomposed as two gradients, and this can be computed numerically efficiently for a general class of convex programs \cite{agrawal2019differentiable}. In our case, we derived the weights analytically by solving the optimization problem:
\begin{align*}
a^*(\mu)
=&\operatorname*{\arg\min}_{a\in\mathbb{R}^N}
\sum_{n=1}^N
\mathbb{E}_{\Xi_n\sim\mathcal{N}(\mu_n,\sigma^2)}
\left[\max\{\Xi_n-a_n-s_n,0\}\right] \\
&\text{subj. to}\quad
a_n\ge0~(\forall n\in[N]),
\quad
\sum_{n=1}^Na_n\le b.
\end{align*}

Letting $\Xi_n=\mu_n+\eta_n$, where $\eta_n\sim\mathcal{N}(0,\sigma^2)$ i.i.d., we can form the Lagrangian:
\begin{align*}
L(a,\lambda) = \sum_{n=1}^{N} \mathbb{E}_{\eta_n} \left[ \max \{\mu_n + \eta_n - a_n - s_n, 0\} \right] + \lambda_0\left( b - \sum_{n=1}^{N} a_n \right) - \sum_{n=1}^N\lambda_n a_n.
\end{align*}
The KKT conditions are
\begin{align*}
0 &= \nabla_{a_n} L(a^*(\mu),\lambda^*(\mu)) 
= -\mathbb{P}_{\eta_n} \left[ a^*_n(\mu) \leq \mu_n + \eta_n - s_n \right] + \lambda_0^*(\mu) - \lambda_n^*(\mu) \\
0 &= \nabla_{\lambda_0} L(a^*(\mu),\lambda^*(\mu)) = \sum_{n=1}^{N} a^*_n(\mu) - b \\
0 &= \lambda_n^*(\mu)a^*_n(\mu) \\
0 &\le \lambda_n^*(\mu) \\
0 &\le a^*_n(\mu).
\end{align*}
Let
\begin{align*}
\mathcal{I}(\mu) = \{ n \in [N] \mid \lambda_n^*(\mu) = 0 \}.
\end{align*}
Note that if $n\not\in\mathcal{I}(\mu)$, then $a_n^*(\mu)=0$, so the first condition becomes
\begin{align*}
a_n^*(\mu) &=
\begin{cases}
\mu_n - s_n + F_{\eta_n}^{-1}(1 - \lambda_0^*(\mu)) &\text{if }n \in \mathcal{I}(\mu), \\
0&\text{otherwise}.
\end{cases}
\end{align*}
where $F_{\eta_n}$ is the CDF of $\eta_n$. Since the $\eta_n$ are i.i.d., we write this CDF as simply $F_\eta$. Summing over $n\in[N]$, we have
\begin{align*}
b = \sum_{n=1}^N a_n^*(\mu) 
&= \sum_{n\in\mathcal{I}(\mu)} \left(\mu_n - s_n + F_\eta^{-1}(1 - \lambda_0^*(\mu))\right) \\
&= \left(\sum_{n\in\mathcal{I}(\mu)}\mu_n - s_n\right) + |\mathcal{I}(\mu)|\cdot F_\eta^{-1}(1 - \lambda_0^*(\mu)).
\end{align*}
Thus, we have
\begin{align*}
F_\eta^{-1}(1 - \lambda_0^*(\mu)) = \frac{b - \sum_{n\in\mathcal{I}(\mu)}(\mu_n - s_n)}{|\mathcal{I}(\mu)|},
\end{align*}
so
\begin{align*}
a_n^*(\mu) &=
\begin{cases}
\mu_n - s_n + \frac{1}{|\mathcal{I}(\mu)|}\left(b - \sum_{m\in\mathcal{I}(\mu)}(\mu_m - s_m)\right) &\text{if }n \in \mathcal{I}(\mu), \\
0&\text{otherwise}.
\end{cases}
\end{align*}
Taking the derivative with respect to $\mu_m$ ($m\in\mathcal{I}(\mu)$), we have
\begin{align*}
\nabla_{\mu_m} a^*_n(\mu) = \delta_{m,n} - \frac{1}{|\mathcal{I}(\mu)|}.
\end{align*}
Thus, we have
\begin{align*}
\nabla_{\mu_m} L(a^*(\mu),\lambda^*(\mu)) &= \nabla_{\mu_m} \sum_{n=1}^{N} \mathbb{E}_{\eta_n} \left[ \max \{\mu_n + \eta_n - a^*_n(\mu) - s_n, 0\} \right] \\
&= \sum_{n=1}^{N} \mathbb{P}_{\eta_n} \left[ a^*_n(\mu) \leq \mu_n + \eta_n - s_n \right] \nabla_{\mu_m} a^*_n(\mu) \\
&= \sum_{n\in\mathcal{I}(\mu)} \mathbb{P}_{\eta_n} \left[ a^*_n(\mu) \leq \mu_n + \eta_n - s_n \right] \left(\delta_{m,n} - \frac{1}{|\mathcal{I}(\mu)|}\right) \\
&= \mathbb{P}_{\eta_m} \left[ a^*_m(\mu) \leq \mu_m + \eta_m - s_m \right] - \frac{1}{|\mathcal{I}(\mu)|} \sum_{n\in\mathcal{I}(\mu)} \mathbb{P}_{\eta_n}[a_n^*(\mu)\le\mu_n+\eta_n-s_n] \\
&= \mathbb{P}_{\eta_m} \left[ a_m^*(\mu) \leq \mu_m + \eta_m - s_m \right] + \text{const} \\
&\approx \mathbb{I}[a_m^*(\mu) \leq \mu_m - s_m] + \text{const},
\end{align*}
where the approximation on the last line holds when $\sigma$ is small. Using this gradient, our predictive model's objective can be approximated as
\begin{align*}
\hat\theta = \operatorname*{\arg\min}_{\theta} \sum_{t=1}^T \sum_{n\in\mathcal{I}(\mu)} (\mathbb{I}[a_{t,n}^*(\mu) \leq \mu_{t,n} - s_{t,n}] + c) \cdot |\mu_{\theta}(x_{t,n}) - \xi_{t,n}^*|.
\end{align*}
for some constant $c$. This loss up-weights the training examples that are more likely to experience unmet demand. One remaining issue is that we do not know $a_{t,n}^*(\mu)$, which is required for the computation of the weight. We approximated it by first training a decision-blind model, using it to optimize the allocations, and then using these allocations to determine the weights.

\subsubsection{Out-of-Sample Comparison}
\label{sup:outofsample}
Prior to deploying our framework, we validated it on historical data by showing that it outperforms several baselines on a held-out test set, including:
\begin{itemize}

\item{\textbf{Excel tool:}} Data-driven tool previously used in conjunction with DHIS2 by the NMSA to make allocations.

\item{\textbf{Decision-blind ablation:}} Follows our pipeline but uses the MSE loss to train the random forest instead of our decision-aware learning algorithm.

\item{\textbf{StochOptForest:}} Follows our pipeline but trains decision-aware random forests using the end-to-end optimization strategy by~\cite{kallus2022stochastic}.

\item{\textbf{Distribution modeling:}} Estimates demand via distribution modeling (as described in \S\ref{sup:multitasklearning}), and then optimizes allocations based on these forecasts.

\item{\textbf{Global Health:}} Estimates demand via a 3-month rolling average (described in \S\ref{sup:multitasklearning}, common in global health~\cite{USAID2014HealthCommodities, Kenya2016FamilyPlanning}), and then optimizes allocations based on these forecasts~\cite{WHO}.

\item{\textbf{Population-based:}} Allocates total available central stock to chiefdoms proportionally to their at-risk population (women and children), as commonly done in global health \cite{USAID2014HealthCommodities}. Within a chiefdom, all facilities are treated equally.\footnote{Note that current data sources do not provide facility-level population estimates in Sierra Leone; we construct such estimates from a variety of data sources in \S\ref{sup:catalyticpriors}.}
\end{itemize}
Here, we introduce the index $m\in[M]$ to denote products. We applied our framework and evaluate the unmet demand for each facility-product pair:
\begin{equation} \label{eq:unmetD}
\text{UnmetDemand}_{m,n} = 
\begin{cases} 
0 & \text{if } \mu_{m,n} - a_{m,n} - s_{m,n} < 0 \\
\mu_{m,n} - a_{m,n} - s_{m,n} & \text{otherwise}
\end{cases}
\end{equation}
for all products $m\in[M]$ and facilities $n\in[N]$, where $a_{m,n}$ is allocation decision and $\mu_{m,n}$ is the true demand. Then, we computed the total unmet demand across all facilities and divided it by total demand to obtain a normalized total unmet demand, and averaged across products:~\footnote{Note that we approximate total demand by total consumption.} 
\begin{align}
\text{NormalizedTotalUnmetDemand}=\frac{1}{\#\text{ products}}\sum_{m=1}^M\frac{\sum_{n=1}^N\text{UnmetDemand}_{m,n}}{\sum_{n=1}^N\mu_{m,n}}.
\end{align}
To test allocation performance in challenging environments, we focused our evaluation on lower budgets---specifically, we used the 25th percentile of quarterly budgets (by product) observed in our data. We then measured how much our approach reduces unmet demand compared to the baseline:
\begin{align}
\frac{\text{NormalizedTotalUnmetDemand}_{\text{Baseline}}-\text{NormalizedTotalUnmetDemand}_{\text{Ours}}}{\text{NormalizedTotalUnmetDemand}_{\text{Baseline}}}.
\end{align}
Our results are shown in Table~\ref{tab:improvement}; they illustrate that our approach outperforms other baselines. First, it performs significantly better than non-machine learning approaches, demonstrating the predictive power of machine learning. Next, StochOptForest most likely performs poorly since it is unable to integrate its decision-aware learning strategy with multi-task learning. At a high level, their algorithm assumes that there is a single optimization problem for each example in the dataset (i.e., one for each facility-product pair). However, in our problem, a single product's optimization problem is associated with many examples (i.e., observations across all facilities). Thus, to apply their approach, we need to decouple the optimization problem into a separate optimization problem for each facility, which we do using the optimal dual variable $\lambda$. However, this eliminates cross-learning between facilities, which may explain the poor results.
Finally, we also demonstrated a 5\% improvement compared to a purely decision-blind approach; while this improvement is comparatively smaller, a 5\% reduction in unmet demand still has significant implications for social welfare. Furthermore, it demonstrates the potential for decision-aware learning to improve performance even compared to powerful state-of-the-art models such as multi-task random forests.

\begin{table}[t]
\centering
\caption{\textbf{Average \% Improvement in Unmet Demand.} We compare our framework vs. various baselines on an out-of-sample historical dataset.}
\begin{tabular}{lr}
\toprule
\multicolumn{1}{c}{Method} & \multicolumn{1}{c}{Improvement \%} \\
\midrule
Our Framework & 0\% \\
Decision-Blind Ablation & 5\% \\
Population Based Census & 27\% \\
Distribution Modeling & 82\% \\
Global Health (3 Month Rolling Avg) & 88\% \\
StochOptForest & 92\% \\
Existing Excel Tool & 98\% \\
\bottomrule
\end{tabular}
\label{tab:improvement}
\end{table}

\section{Evaluation and Deployment}
\label{sup:evaldeploy}

In this section, we describe how we deployed our system (\S\ref{sup:deployment}) and econometrically evaluated its impact using SynthDiD (\S\ref{sup:mainanalysis}); we established an equivalence between consumption and unmet demand to justify this analysis (\S\ref{sup:consumption}). Then, we investigated real-world compliance to our allocations (\S\ref{sup:compliance}), analyzed the distributional impacts of our deployment (\S\ref{sup:heteroanalysis}), and performed multiple robustness checks to validate our main findings (\S\ref{sup:robustnesschecks}). Finally, we performed a cost-effectiveness analysis (\S\ref{sup:cost}).

\subsection{Deployment Details}
\label{sup:deployment}

The Sierra Leone national government deployed our system in 2023 Q2 for 5 randomly selected districts: Tonkolili, Falaba, Karene, Kono, and Pujehun; subsequently, it was rolled out nationwide in 2023 Q3. We present a balance table for pre-treatment covariates in Table~\ref{tab:balance_district};\footnote{This table is based on district-level covariates from the Global Data Lab~\cite{GlobalDataLab} and the Service Delivery Indicators Health Survey~\cite{ghndr2018sierra}. We constructed the corresponding treatment and control values by averaging over districts in each respective group, weighted by their population. One issue is that the districts in both data sources predated the 2017 creation of two new districts (Koinadugu split in two into Falaba and Koinadugu; Bombali split in two into Karene and Bombali); Falaba and Karene are treated districts whereas Koinadugu and Bombali are control districts. We took the covariates for Falaba and Karene to be the same as the ones for Koinadugu and Bombali, respectively.} as can be seen, there are no statistically detectable differences between the two groups. Furthermore, we present data missingness rates for treated and control districts in Table~\ref{tab:balanceMiss}. Prior to the deployment, the government had established predetermined supply levels for this period and allocated resources to control districts. The remaining supply was then assigned to the treatment districts, maintaining independence of supply quantities between the two groups. 

Before the implementation, we first conducted two training sessions for policymakers and frontline workers to provide them with a technical understanding of how our tool operated and what it did. We ensured that our allocation tool was compatible with the same formatted inputs and outputs as the prior Excel allocation tool, allowing users to maintain their existing workflows with minimal changes. Fig.~\ref{fig:WebApp} shows a screenshot of our web interface.

\begin{figure}[t]
\centering
\includegraphics[width=.8\linewidth]{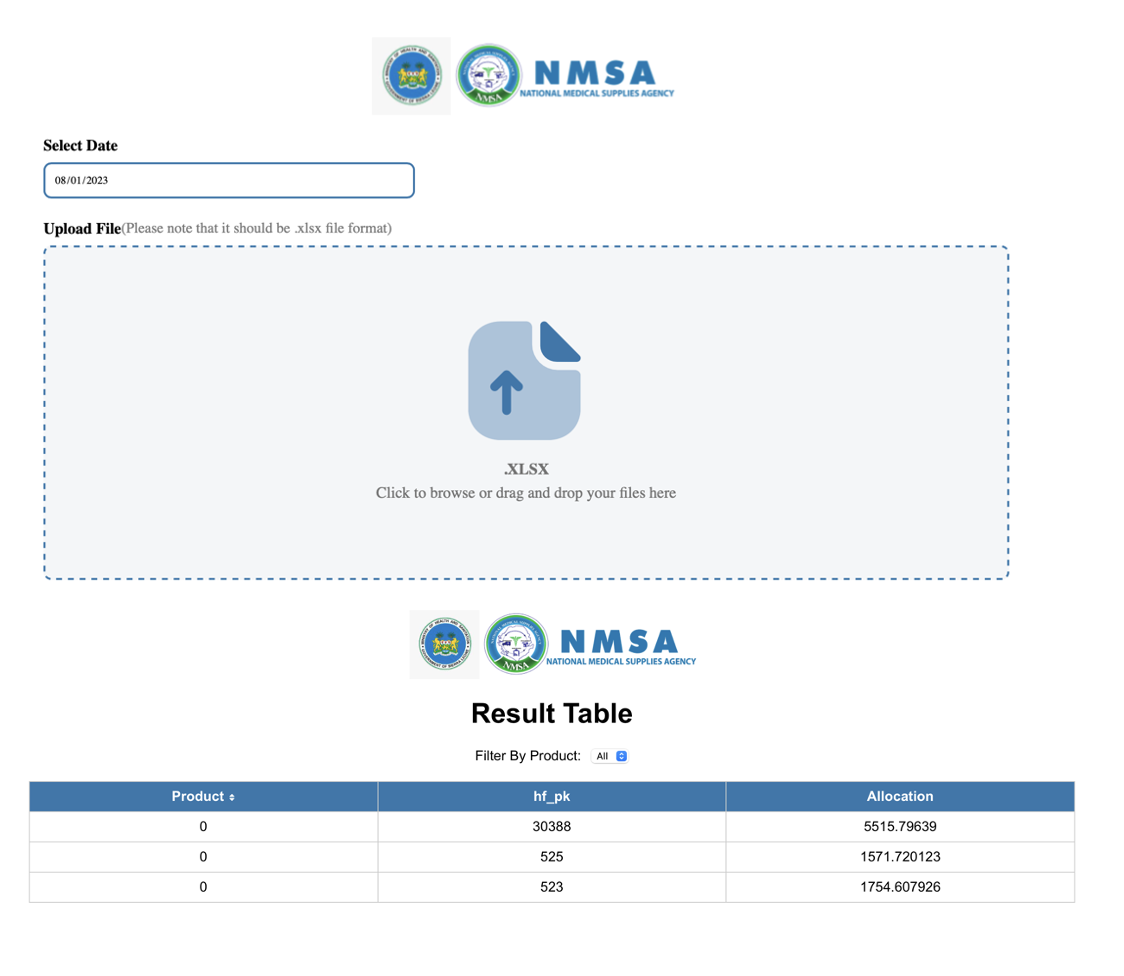}
\caption{\textbf{Our System's Web App Interface}: Consistent with their prior workflow, users upload an Excel sheet with total available central stock information (which is directly downloaded from the mSupply database) and then run the tool to obtain downloadable allocation results. This web app is now owned and operated by the Sierra Leone national government.}
\label{fig:WebApp}
\end{figure}

The deployment began in June 2023. The implementation timeline proceeded as follows: in the Falaba and Tonkolili districts, last-mile delivery to local health facilities was completed by mid-July, while in the Karene, Kono, and Pujehun districts, implementation extended to the end of July due to logistical delays caused by the presidential election in Sierra Leone. To evaluate the impact of the intervention, we primarily analyzed outcomes from 2023 Q2 to Q3. The government did not conduct an allocation in Q4 due to insufficient supply, arising from a central-level shortage of health aid from donors.

Our system functioned as a decision support tool, where central and district-level planners retain ultimate authority and can override the system when necessary. During deployment, when planners identified discrepancies between our recommendations and their practices (e.g., due to data lags, or a strategic need to maintain a larger safety stock for certain items that quarter), they were able to easily adjust the allocations. This human-in-the-loop approach was crucial for building trust and enabling context-specific adaptability. We found that decision-makers closely followed the algorithmic allocations---the normalized overlap between the actual and algorithmic allocations ranged from 0.89 to 1 (see \S\ref{sup:compliance} for details). 
The system automatically re-trains the machine learning model based on updated data for each new quarterly allocation, so that it continuously adapts to shifts in demand patterns. The system is hosted on AWS with government-controlled access, ensuring institutional ownership. We provided knowledge transfer to local personnel to build internal capacity for system maintenance.

\subsection{Main Analysis Methodology}
\label{sup:mainanalysis}

While the raw data in Fig.~\ref{fig:raw} shows a descriptive increase in consumption, the small number of treated districts necessitates an evaluation strategy that appropriately matches pre-treatment patterns across treated and control groups. To this end, we used SynthDiD~\cite{arkhangelsky2021synthetic} to analyze the impact of our deployment on patient consumption. SynthDiD identifies causal effects by ensuring that the difference between treated and synthetic control facilities remains stable before treatment. To achieve this goal, SynthDiD assigns two sets of weights: one for control facilities and another for time periods. The control facility weights are chosen so the weighted average of control facilities' outcomes best match the unweighted average of treated facilities' outcomes in the pre-treatment period. The time period weights ensure that the weighted average of the control facilities' outcomes in the pre-treatment period best match their unweighted average of the control facilities' outcomes in the post-treatment period. 

By combining these weights, SynthDiD constructs synthetic control facilities whose pre-treatment trends align with those of the treated facilities, thereby providing a credible estimate of the treatment’s causal impact and addressing key limitations in other commonly used estimators. Unlike standard difference-in-differences (DiD), which requires a parallel trends assumption, SynthDiD remains robust even when treatment and control groups show different trends before the intervention. Furthermore, it can effectively control for variations in outcomes that arise from both time-related and unit-specific factors, unlike standard synthetic controls, which requires a near-perfect match in pre-treatment levels.

Using unit weights $\hat{\omega}_n^{\text{sdid}}$ and time weights $\hat{\lambda}_t^{\text{sdid}}$ derived from Eqs.~(4) \&~(6) in~\cite{arkhangelsky2021synthetic}, the average effect of the treatment on the treated (ATT) is estimated as follows:
\begin{equation}
\left( \hat{\tau}^{\text{sdid}}, \hat{\mu}, \hat{\alpha}, \hat{\beta} \right) = 
\underset{\tau, \mu, \alpha, \beta}{\operatorname{argmin}} \;
\sum_{n=1}^N \sum_{t=1}^T 
\left( Y_{n,t} - \mu - \alpha_n - \beta_t - W_{n,t} \tau \right)^2 
\hat{\omega}_n^{\text{sdid}} \hat{\lambda}_t^{\text{sdid}}
\end{equation}
where the outcome $Y_{n,t}$ is the average consumption for facility $n$ at time $t$ across all allocated products; $\mu$ is the baseline average outcome, $\alpha_n$ are facility fixed effects, and $\beta_t$ are time period fixed effects. The treatment assignment $W_{n,t}$ indicates whether facility $n$ received treatment at time $t$, and $\tau$ is the treatment effect. Control facility weights $\hat{\omega}_n^{\text{sdid}}$ and time period weights $\hat{\lambda}_t^{\text{sdid}}$ account for differences between treated and control groups over time.

We independently performed our analysis using both the \textit{synthdid} package in R and the \textit{sdid} package in STATA, finding consistent results. We estimated standard errors using the jackknife (Algorithm~3 in \cite{arkhangelsky2021synthetic}). To validate our use of SynthDiD, we also performed an event study~\cite{clarke2023synthetic}, which showed that there are no statistically significant differences between treated facilities and control facilities prior to our intervention, and that the change in consumption emerged only after our system was deployed (see Fig.~\ref{fig:Main_Event}). We used Eq.~(8) in \cite{clarke2023synthetic}, which compares the treated-minus-synthetic-control difference in each time period $t$ to a baseline pre-treatment difference:
\begin{align*}
(\bar{Y}_t^{Tr} - \bar{Y}_t^{Co}) - (\bar{Y}_{\text{baseline}}^{Tr} - \bar{Y}_{\text{baseline}}^{Co}),
\end{align*}
where $\bar{Y}_{\text{baseline}}^{Tr}$ and $\bar{Y}_{\text{baseline}}^{Co}$ are the baseline means for the treated group and the synthetic control group, respectively. Unlike conventional event studies that choose a single pre-treatment period as the baseline, the SynthDiD framework selects the optimal pre-treatment weights $\hat{\lambda}_t^{sdid}$
\begin{align*}
\bar{Y}_{\text{baseline}}^{Tr} = \sum_{t=1}^{T_{\text{pre}}} \hat{\lambda}_t^{sdid}\bar{Y}_t^{Tr}
\quad \text{and} \quad
\bar{Y}_{\text{baseline}}^{Co} = \sum_{t=1}^{T_{\text{pre}}} \hat{\lambda}_t^{sdid}\bar{Y}_t^{Co}.
\end{align*}
Then, we constructed the event-study estimates from SynthDiD using the following steps:
\begin{enumerate}
\item \textbf{Initial Estimation:}
\begin{enumerate}
\item Fit SynthDiD on the full sample to obtain time period weights $\hat{\lambda}$ and the pre-treatment difference in outcomes.
\item Adjust post-treatment differences by subtracting the pre-treatment mean difference.
\end{enumerate}
\item \textbf{Bootstrap Inference:}
\begin{enumerate}
\item For $b \in\{ 1,\ldots,B\}$, resample the data and re-estimate SynthDiD.
\item Compute the bootstrap-adjusted difference series for each replication.
\item Use these bootstrap replicates to compute standard errors and form confidence intervals.
\end{enumerate}
\end{enumerate}

\begin{figure}[t]
\centering
\includegraphics[width=0.5\linewidth]{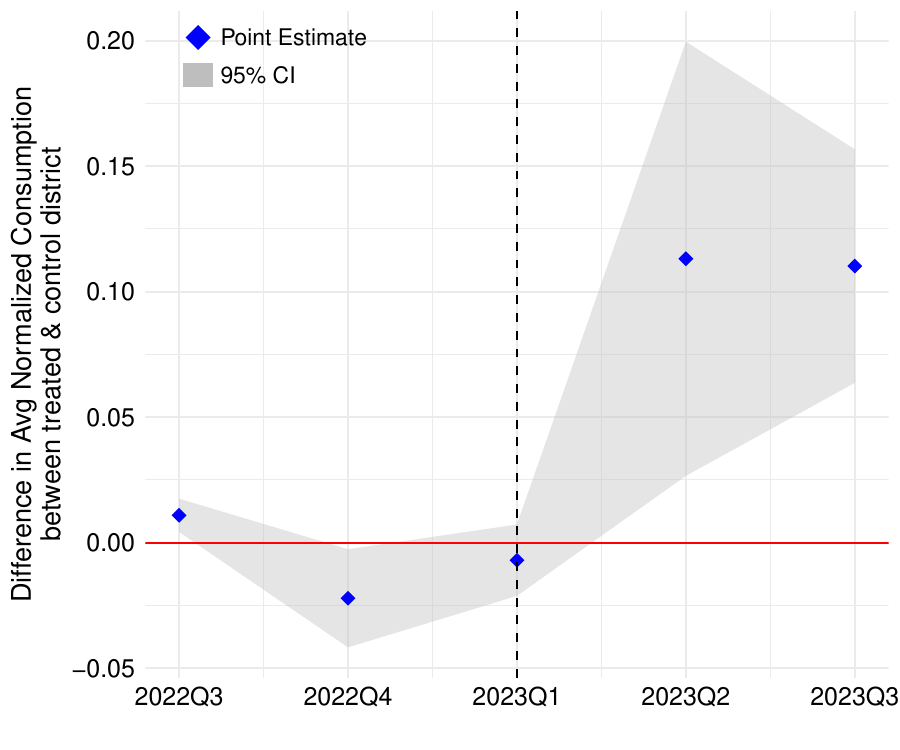}
\caption{\textbf{SynthDiD Event Study.} This plot shows the estimated ATT across time. Blue diamonds represent point estimates, and the shaded region represents 95\% confidence intervals. The dashed vertical line indicates the time period immediately prior to our 2023 Q2 deployment. As expected, the pre-treatment estimates are close to zero, while the post-treatment estimates are positive and large, indicating that our deployment increased consumption.}
\label{fig:Main_Event}
\end{figure}

\subsection{Equivalence of Consumption and Unmet Demand}
\label{sup:consumption}

One of the key challenges in evaluating our deployment is that unmet demand---the usual metric for evaluating the efficiency of inventory management policies---is censored since we do not observe when patient orders go unfulfilled. Instead, we use consumption as our main metric, which measures the total amount of each product distributed to patients across all facilities. First, we show that under a simple model of dispensing, maximizing total consumption and minimizing total unmet demand are mathematically equivalent in our optimization formulation. Next, we show that this equivalence is robust to stock-dependent dispensing (such as rationing or over-dispensing) and product substitution.

\paragraph{Basic Model.} In each quarter $t$, each facility $n\in[N]$ has current stock $s_{m,n,t}\in\mathbb{R}_{\ge0}$ of product $m\in[M]$ and needs to fulfill demand $\xi_{m,n,t}\in\mathbb{R}_{\ge0}$ until it runs out of stock. Our goal is to allocate medicines to each facility to minimize unmet demand. The unmet demand and consumption for product $m$ and facility $n$ in time period $t$ as a function of the allocation $a\in\mathbb{R}_{\ge0}$ of that product to that facility in that time period are
\begin{align*}
U_{m,n,t}(a)&=\max\{\xi_{m,n,t}-(s_{m,n,t}+a),0\} \\
C_{m,n,t}(a)&=\min\{\xi_{m,n,t},s_{m,n,t}+a\},
\end{align*}
respectively.
It is easy to check that
\begin{align*}
U_{m,n,t}(a)+C_{m,n,t}(a)=\xi_{m,n,t}
\end{align*}
for any allocation $a$. Given allocation decisions $\{a_n\}_n$ for each facility $n\in[N]$, we have
\begin{align*}
U_{m,t}(\{a_n\}_n)+C_{m,t}(\{a_n\}_n)=\xi_{m,t},
\end{align*}
where $U_{m,t}(\{a_n\}_n)=\sum_{n=1}^NU_{m,n,t}(a_n)$ is total unmet demand, $C_{m,t}(\{a_n\}_n)=\sum_{n=1}^NC_{m,n,t}(a_n)$ is total consumption, and $\xi_{m,t}=\sum_{n=1}^N\xi_{m,n,t}$ is the total demand for quarter $t$. Thus, total unmet demand plus total consumption does not depend on our allocation decisions $\{a_n\}_n$; therefore, the allocation that minimizes unmet demand also maximizes consumption.

Of course, this model makes a number of strong assumptions about how stock is distributed. We describe how to handle alternative behaviors: (1) rationing (i.e., when the facility is low on stock, it allocates less to patients than their need), (2) over-dispensing (i.e., when the facility has excess stock, it allocates more to patients than they need), and (3) substitution (i.e., when stock of one medication runs out, the facility allocates a substitute). We discuss rationing and over-dispensing in \S\ref{sup:rationing}, and substitution in \S\ref{sup:substitution}.

\subsubsection{Rationing and Over-Dispensing.}
\label{sup:rationing}

Providers may change their dispensing patterns dynamically as a function of the available stock remaining. For instance, if product stock is running low and the next allocation isn't anticipated for some time, a provider may ration stock by only allocating to high-risk patients; alternatively, if product stock is relatively high and the next allocation is anticipated soon, a provider may dispense to even low-risk patients who are not normally allocated. We show using a Markov Decision Process analysis that any such stock-dependent allocations still satisfy the equivalence of consumption and demand, as long as providers do not \textit{waste} stock (i.e., dispense to a patient with no demand, or dispense more than they could plausibly use). Waste is unlikely to occur in our setting due to the limited supply. We also performed an empirical analysis to try to detect rationing and over-dispensing behavior in our data; we found no evidence for such behaviors.

\paragraph{Theoretical analysis.} For this section, we redefine the variables $t$, $s_t$, $a_t$, and $\xi_t$ to be the index of the current patient, the current stock of a medicine at a given facility when patient $t$ arrives, the amount of medicine distributed to patient $t$, and the target demand of patient $t$, respectively. Specifically, we consider a model where (1) patients arrive at the facility sequentially over $T$ steps (indexed by $t\in[T]$), are distributed $a_t$ of the product and then leave; (2) each patient has a target demand $\xi_t$ depending on their type $x_t\in[K]$, and the facility distribution satisfies $a_t\le\xi_t$; and (3) unmet demand for that patient is $\xi_t-a_t$. Under this assumption, we prove that unmet demand and negative consumption are equivalent even in this general setting:
\begin{align*}
\bar{U}(\pi)+\bar{C}(\pi)=\text{const},
\end{align*}
where $\bar{U}(\pi)$ and $\bar{C}(\pi)$ are the expected unmet demand and consumption when using policy $\pi$, respectively. This result establishes under general conditions that consumption is a reasonable proxy for unmet demand.

We model the system as a finite-horizon MDP $(S,A,R,P,T)$, where the state is a pair $(s,x)\in S=(\mathbb{N}\cup\{0\})\times[K]$ consisting of the current stock $s$ and the type $x$ of the patient being served on step $t$, the action $a\in A=\mathbb{N}$ indicates the amount of medication to dispense to the patient, the reward $R(s,x,a)\ge0$ indicates the payoff for dispensing medicine to the patient (with $R(s,x,0)=0$), the transition measure $P(s,x,s',x')=\mathbbm{1}(s'=s-a)\cdot P(x')$ consists of a deterministic update $s'=s-a$ to the stock and an i.i.d. sample of the next patient type $x'$, and the time horizon is $T\in\mathbb{N}$. We assume that patient type $x=1$ indicates no patient, and $R(s,1,a)=0$ for all $a\in A$.
We impose the constraint that $a\le s$. We let $\xi(x)\in\mathbb{N}$ be the ideal amount of medication needed by patient of type $x$.

We assume the facility is using a state-dependent and time-dependent policy $\pi_t:S\to A$ to distribute medicine to patients. Since the state includes the current stock and the time encodes the time remaining until the end of the quarter, this policy can encode very complex rationing and over-dispensing strategies. As usual, a given policy $\pi$ induces a distribution $D^{(\pi)}$ over rollouts $\zeta=((s_1,x_1,a_1,r_1),...,(s_T,x_T,a_T,r_T))$. We let $\bar{U}(\pi)=\mathbb{E}_{\zeta\sim D^{(\pi)}}[U(\zeta)]$ and $\bar{C}(\pi)=\mathbb{E}_{\zeta\sim D^{(\pi)}}[C(\zeta)]$ denote the expected unmet demand and consumption, respectively. The following result shows that under this general model, unmet demand and negative consumption are equivalent metrics.

\begin{theorem}
Assume $a_t\le\xi_t$ (where $\xi_t=\xi(x_t)$) for all $t\in[T]$. Then, we have $\bar{U}(\pi)=\bar{\Xi}-\bar{C}(\pi)$, where $\bar{\Xi}$ is a constant independent of $\pi$.
\end{theorem}

\begin{proof}
Given an arbitrary rollout $\zeta=((s_1,x_1,a_1,r_1),...,(s_T,x_T,a_T,r_T))$, the total unmet demand across $\zeta$ is
\begin{align*}
U(\zeta)=\sum_{t=1}^T\max\{\xi_t-a_t,0\}\cdot\mathbbm{1}(x_t\neq1),
\end{align*}
where $\xi_t=\xi(x_t)$, and consumption is
\begin{align*}
C(\zeta)=\sum_{t=1}^T\min\{a_t,\xi_t\}\cdot\mathbbm{1}(x_t\neq1).
\end{align*}
Under our assumption that $a_t\le\xi_t$, we have
\begin{align*}
U(\zeta)
=\sum_{t=1}^T(\xi_t-a_t)\cdot\mathbbm{1}(x_t\neq1)=\Xi(\zeta)-\sum_{t=1}^Ta_t\cdot\mathbbm{1}(x_t\neq1),
\end{align*}
and
\begin{align*}
C(\zeta)
=\sum_{t=1}^Ta_t\cdot\mathbbm{1}(x_t\neq1),
\end{align*}
where $\Xi(\zeta)=\sum_{t=1}^T\xi_t$; it follows that
\begin{align*}
U(\zeta)=\Xi(\zeta)-C(\zeta).
\end{align*}
Taking the expectation over $\zeta\sim D^{(\pi)}$, where $D^{(\pi)}$ is the distribution over rollouts induced by using policy $\pi$, we have
\begin{align*}
\bar{U}(\pi)=\bar{\Xi}-\bar{C}(\pi),
\end{align*}
where $\bar{\Xi}=\bar{\Xi}(\pi)=\mathbb{E}_{\zeta\sim D^{(\pi)}}[\Xi(\zeta)]$ is the expected total demand; note that $\Xi(\zeta)$ only depends on $\zeta$ through the sequence of patient covariates $(x_1,...,x_T)$, so $\bar{\Xi}(\pi)$ is independent of $\pi$.
\end{proof}
\paragraph{Empirical analysis.}

Next, we performed empirical analyses to try to detect rationing and over-dispensing behaviors. Ideally, we would have patient-level data to understand if providers are allocating differently to different risk profiles as a function of available stock; however, such data does not exist in a digitized, unified form in many developing countries due to infrastructure limitations. As a result, we are limited to analyzing aggregate monthly allocations.

First, we considered rationing---facilities with low inventory may preemptively reduce consumption (e.g., by allocating only to high-risk patients) even if they are sufficiently well stocked, in order to avoid a stockout in future months before the next allocation. Consider a facility in month $t-2$ (with an upcoming allocation in month $t$) that has sufficient stock for the current month $t-2$ but insufficient stock for the coming month $t-1$. If we see ``less than expected'' consumption in $t-2$, that may indicate rationing behavior.

To test this, we first identified facility-month pairs $(n,\tau)$ in our panel data where the opening balance is sufficiently large to likely avoid a stockout in that month.\footnote{If not, reduced consumption in that month would simply reflect a stockout rather than rationing behavior to protect stock for future months.} Specifically, we only considered $(n,\tau)$ pairs such that the initial stock $s_{n,\tau}$ exceeded the average consumption $\bar{Y}_n$ by some multiplicative constant, i.e., $s_{n,\tau}>C_2\bar{Y}_n$ for some $C_2\ge1$. Then, we estimated the following regression model:
\begin{align}
\label{eqn:rationingspec}
Y_{n,\tau} &= \beta_1 \text{LowStock}_{n,\tau} + \beta_2 \text{TimeIndicator}_\tau
+ \delta (\text{LowStock}_{n,\tau} \times \text{TimeIndicator}_\tau) \nonumber\\
&\qquad+ \alpha_n + \lambda_\tau + \gamma X_{n,\tau} + \epsilon_{n,\tau}
\end{align}
where $Y_{n,\tau}$ is the consumption of facility $n$ in month $\tau$; $\alpha_n$ and $\lambda_\tau$ are facility and time fixed effects, respectively; $X_{n,\tau}$ denotes controls (including all the features in our demand prediction model); $\text{LowStock}_{n,\tau}=\mathbbm{1}(s_{n,\tau}\le C_1\bar{Y}_n)$ for some $C_1 < 2$ indicates whether facility $n$ has likely insufficient opening balance in month $\tau$ to satisfy future demand in month $\tau + 1$, as estimated by a multiplier of its average demand $\bar{Y}_n$; and $\text{TimeIndicator}_{\tau}=\mathbbm{1}(\exists t\in\mathcal{T}_{\text{alloc}}\text{ s.t. }\tau=t-2)$ indicates whether $\tau$ is two months prior to a planned allocation $t\in\mathcal{T}_{\text{alloc}}$. The coefficient of interest is $\delta$; a negative value of $\delta$ would indicate that facilities reduce distribution in months where they have low stock in month $t-2$ to save for month $t-1$, where $t$ is an allocation month.

For robustness, we considered several choices of our constants---specifically, $C_1\in\{1.5,1.75\}$ and $C_2\in\{1.0,1.5\}$; note that we require $C_1 > C_2$. Results from these regressions are shown in Table~\ref{tab:rationing_interaction}. As can be seen, $\delta$ is not statistically significantly different from zero in any of them, so we found no evidence of rationing behaviors. We note that this empirical analysis is only suggestive, as we do not have the data to support a causal or more granular analysis.

\begin{table}[t]
\caption{\textbf{Analysis of Rationing Behavior.}}
\label{tab:rationing_interaction}
\centering
\begin{threeparttable}
\begin{tabular}{lcc}
\toprule
& \multicolumn{2}{c}{Condition on Opening Balance $C_2$} \\
\cmidrule(lr){2-3}
Initial Supply Level $C_1$ & $C_2=1.0$ & $C_2=1.5$ \\
\midrule
$C_1=1.5$ & 0.986 & -- \\
& (0.592) & \\
\addlinespace
$C_1=1.75$ & 0.360 & $-$0.822 \\
& (0.496) & (0.810) \\
\bottomrule
\end{tabular}
\begin{tablenotes}
[flushleft]
\footnotesize
\item \textit{Notes:} The table reports regression coefficients of the interaction term $\delta$. Standard errors are in parentheses: $^{*} p < 0.05$, $^{**} p < 0.01$.
\end{tablenotes}
\end{threeparttable}
\end{table}

Next, we considered over-dispensing---facilities with high inventory may preemptively increase consumption (e.g., by allocating to low-risk patients) in anticipation of an upcoming allocation. Consider a facility in month $t-1$ (with an upcoming allocation in month $t$) that has more than sufficient stock for the current month. If we see ``more than expected'' consumption in $t-1$, that may indicate a facility trying to over-dispense a surplus.

To test this, as before, we identified facility-month pairs in our panel data where the opening balance is sufficiently large to avoid a stockout in that month, i.e., $(n,\tau)$ pairs such that $s_{n,\tau}>C_2\bar{Y}_n$.\footnote{If not, increased consumption in that month would simply reflect the lack of a stockout rather than over-dispensing behavior to use up surplus stock.} Then, we estimated the following regression model:
\begin{align}
\label{eqn:excessspec}
Y_{n,\tau} &= \beta_1 \text{ExcessStock}_{n,\tau} + \beta_2 \text{TimeIndicator}_\tau
+ \delta (\text{ExcessStock}_{n,\tau} \times \text{TimeIndicator}_\tau) \nonumber\\
&\qquad+ \alpha_n + \lambda_\tau + \gamma X_{n,\tau} + \epsilon_{n,\tau}
\end{align}
where $\text{ExcessStock}_{n,\tau}=\mathbbm{1}(s_{n,\tau}\ge C_1\bar{Y}_n)$ for some $C_1 > 2$ indicates whether facility $n$ will likely have a significant surplus in month $\tau+1$ (given the upcoming allocation), as estimated by a multiplier of its average demand $\bar{Y}_n$; we re-defined $\text{TimeIndicator}_{\tau}=\mathbbm{1}(\exists t\in\mathcal{T}_{\text{alloc}}\text{ s.t. }\tau=t-1)$ to indicate whether $\tau$ is one month prior to a planned allocation $t\in\mathcal{T}_{\text{alloc}}$.

For robustness, we considered $C_1\in\{3.0,5.0\}$ and $C_2\in\{1.0,1.5\}$. Results from the resulting regressions are shown in Table~\ref{tab:oversupply_interaction}. As can be seen, $\delta$ is not statistically significantly different from zero in any of these regressions, so we found no evidence of over-dispensing behaviors. We note that this empirical analysis suffers from the same limitations noted earlier for rationing.

\begin{table}[t]
\caption{\textbf{Analysis of Over-Dispensing Behavior.}}
\label{tab:oversupply_interaction}
\centering
\begin{threeparttable}
\begin{tabular}{lcc}
\toprule
& \multicolumn{2}{c}{Condition on Opening Balance $C_2$} \\
\cmidrule(lr){2-3}
Initial Supply Level $C_1$ & $C_2=1.0$ & $C_2=1.5$ \\
\midrule
$C_1=3.0$ & 0.174 & 0.146 \\
     & (0.371) & (0.408) \\
\addlinespace
$C_1=5.0$ & 0.489 & 0.503 \\
     & (0.368) & (0.380) \\
\bottomrule
\end{tabular}
\begin{tablenotes}
[flushleft]
\footnotesize
\item \textit{Notes:} The table reports regression coefficients of the interaction term. Standard errors are in parentheses: $^{*} p < 0.05$, $^{**} p < 0.01$.
\end{tablenotes}
\end{threeparttable}
\end{table}

One concern could be that our estimates are under-powered. Thus, as a robustness check, we took the point estimates for rationing and over-dispensing at face value, and evaluated their impact on our main analysis. Specifically, we took a conservative approach (i.e., assume these behaviors only adversely affect consumption in treated facilities and do not affect control facilities), which translated into an approximately 2 percentage point reduction in consumption in treated facilities. Correspondingly, we adjusted our evaluation by reducing only the treated facilities’ consumption in the post-treatment period by 2 percentage points. We then re-ran our main analysis and found that our deployment still improved consumption by a statistically significant 16\%. This improvement is quite similar to the 19\% improvement found in our original analysis.

\subsubsection{Substitution}
\label{sup:substitution}

Another potential behavior we do not handle in our basic model is cross-product substitution. We show theoretically that consumption and unmet demand still remain equivalent under ``perfect'' product substitution, as long as we evaluate consumption at the level of product clusters (that mutually substitute for each other) with appropriate normalization, instead of evaluating consumption at the level of individual products. We first formalize and prove this result. Based on input from medical experts and field workers, we developed a list of all substitutable products among the ones we allocated; many of these can be considered perfect (e.g., medications with different dosages), for the remaining, perfect substitution was a reasonably good approximation. Then, we performed a robustness check where we re-ran our main analysis using the above strategy and obtained similar results as our main analysis.

\paragraph{Theoretical analysis.}

We consider substitution under the following assumptions: (1) substitutions are not case-dependent (i.e., if two medicines can either always be substituted, or never be substituted), (2) substitutions are transitive (i.e., if $m_1$ substitutes for $m_2$ and $m_2$ substitutes for $m_3$, then $m_1$ substitutes for $m_3$), and (3) substitutions are always at the same rate (i.e., if $m_1$ and $m_2$ are substitutes, then if a patient requires $a$ units of $m_1$, then they require $c_{m_1\to m_2}\cdot a$ units of $m_2$, where $c_{m_1\to m_2}$ is constant across patients (with $c_{m\to m}=1$)). While these assumptions are somewhat strong, they are accurate to first order for the products distributed by our system.

With these assumptions, we can straightforwardly adapt our analysis to handle substitutions. Specifically, rather than consider products separately, we first cluster products into $K$ groups $M_k\subseteq[M]$ that are mutually substitutable (which is possible by Assumptions~1 \&~2). For each group $k\in[K]$, we can choose one focal product $m\in M_k$ in that group and standardize the other products relative to $m$ to obtain the following aggregated stock, allocation, and demand:
\begin{align*}
\tilde{s}_{k,n,t}&=\sum_{m'\in M_k}c_{m'\to m}\cdot s_{m',n,t} \\
\tilde{a}_{k,n,t}&=\sum_{m'\in M_k}c_{m'\to m}\cdot a_{m',n,t} \\
\tilde\xi_{k,n,t}&=\sum_{m'\in M_k}c_{m'\to m}\cdot\xi_{m',n,t}.
\end{align*}
To account for substitution, we can simply analyze consumption using the normalized variables instead of the original ones.

\paragraph{Empirical analysis.}

Using the strategy outlined above, we conducted a robustness check of our main analysis based on the normalized product consumption to account for potential substitution behavior; we discuss this analysis in detail in \S2.6.5. We found a consistent, statistically significant 18\% increase in the consumption of these normalized product groups, demonstrating that substitution is unlikely to be the primary driver behind our treatment effect.

\subsection{Compliance Analysis}
\label{sup:compliance}

Our system was implemented as a decision support tool that could be overridden by district pharmaceutical managers, who may selectively override our recommendations to address shortcomings in our allocation decisions to improve performance (though they may also make mistakes that lead to worse outcomes). To measure compliance, we manually collected the local pharmacist's allocation document and cross-checked all the invoices pulled from mSupply. To quantify compliance of a treated district, we normalized the allocation quantity for each facility-product pair in that district, calculated the absolute difference between the actual allocation and our deployed suggestion, and then averaged this value across facility-product pairs. Results are shown in Table~\ref{tab:compliance}; we found that compliance was generally very high, with the Kono and Pujehun districts having slightly lower compliance than the other three districts. This is likely due to logistical and communication issues that arose during the implementation of our system in 2023 Q2, which were resolved soon after. In \S\ref{ssec:LATE}, we disentangled the effect of compliance with our recommendations in 2023 Q2 from the impact of our system through a standard instrumental variables analysis to compute the Local Average Treatment Effect (LATE). We found a statistically significant 37\% improvement in consumption among compliers, suggesting that compliance with our system's allocations was generally beneficial.

\begin{table}[t]
\centering
\caption{\textbf{Compliance of Treated Districts in 2023 Q2}}
\begin{tabular}{lc}
\toprule
District & Normalized Avg Absolute Diff \\
\midrule
Tonkolili & 0.000 \\
Falaba    & 0.028 \\
Karene    & 0.039 \\
Kono      & 0.073 \\
Pujehun   & 0.109 \\
\bottomrule
\end{tabular}
\label{tab:compliance}
\end{table}

\subsection{Distributional Impact Analysis}
\label{sup:heteroanalysis}

To understand the distributional impacts of our allocation system, we conducted additional analyses across several dimensions.

The first dimension focuses on facility type; given limited resources, NMSA policymakers emphasized prioritizing larger facilities (Hospitals and Community Health Centers (CHCs)) to ensure that care remains available to the greatest number of people. We checked whether these facilities are correctly prioritized by our system. Specifically, we categorized facilities into two groups: (1) larger, better equipped facilities (Hospitals and Community Health Centers (CHCs)), and (2) smaller, community-based facilities (Community Health Posts (CHPs) and Maternal and Child Health Posts (MCHPs)). Then, we conducted our main analysis separately on each group; results are shown in Table~\ref{tab:heterogeneity_full}. As can be seen, our positive treatment effect remains strong and statistically significant for larger facilities, suggesting that our algorithm effectively improves consumption at facilities that are well-equipped and consistently operational (consistent with NMSA's stated priorities). In contrast, smaller health posts show a smaller, statistically insignificant improvement. They may face more fundamental challenges (e.g., inadequate staffing or equipment) that limit their ability to provide services or utilize medical supplies, regardless of availability. For instance, during a field visit, we found a small health post closed for two consecutive days, and local residents shared that they often seek essential care at larger, more reliable facilities. Improving allocation efficiency for larger facilities is desirable since these facilities are relied on to provide consistent care; furthermore, we did so without impacting (or even slightly improving) efficiency for smaller facilities.

Next, we performed two analyses to better understand the equity implications of our system. First, we examined the impact on previously under-served facilities; specifically, a facility is \emph{under-served} if it experienced at least one stockout in the data from 2020 until the quarter before our deployment. Our results are shown in Table~\ref{tab:heterogeneity_full}; as can be seen, our results are substantially larger (compared to our overall effect of 19\%) and statistically significant for under-served facilities. Second, we examined distributional impacts on urban vs. rural facilities, based on geographic location and accessibility; in Sierra Leone, rural facilities tend to be poorer than urban ones. We categorized facilities as urban if they were within a walkable distance of a city or town~\cite{simplemaps_worldcities_2025} (either within 0.25 miles based on \cite{yang2012walking}, or within 1 mile as a robustness check) and rural otherwise. As can be seen from Table~\ref{tab:heterogeneity_full}, the treatment effect is positive and statistically significant for rural facilities (it is also positive but statistically insignificant for urban facilities, but this may be due to the smaller sample size). This suggests that our allocation system succeeded at reducing unmet demand at rural (poor) facilities, likely improving equity.

\begin{table}[t]
\centering
\caption{\textbf{Distributional Analysis Results.} Columns 1-2 compare large facilities (Hospitals and CHCs) vs. smaller, community-based facilities (CHPs and MCHPs). Columns 3-6 analyze rural vs. urban facilities based on distance to a city (0.25 or 1 mile). Column 7 examines under-served facilities that have experienced at least one stockout since 2020. We found that our deployment increased consumption in large facilities, rural facilities, as well as under-served facilities.}
\label{tab:heterogeneity_full}

\begin{threeparttable}
\scriptsize
\setlength{\tabcolsep}{4pt}
\renewcommand{\arraystretch}{0.9}

\resizebox{\textwidth}{!}{
\begin{tabular}{lccccccc}
\toprule
\textit{Dep. var.: Normalized Consumption} & \multicolumn{2}{c}{By Facility Type} & \multicolumn{4}{c}{By Location} & \multirow{4}{*}{\begin{tabular}[c]{@{}c@{}}Under-\\ Served\end{tabular}} \\
\cmidrule(lr){2-3} \cmidrule(lr){4-7}
& Large & Small & \multicolumn{2}{c}{0.25 miles} & \multicolumn{2}{c}{1 mile} &  \\
\cmidrule(lr){4-5} \cmidrule(lr){6-7}
& Facility & Facility & Rural & Urban & Rural & Urban &  \\
& (1) & (2) & (3) & (4) & (5) & (6) & (7) \\
\midrule
Treatment Effect & 0.368$^{*}$ & 0.052 & 0.107$^{**}$ & 0.291 & 0.090$^{*}$ & 0.334 & 0.182$^{**}$ \\
& (0.183) & (0.037) & (0.043) & (0.395) & (0.044) & (0.315) & (0.069) \\
[1ex]
Observations & 1,075 & 4,215 & 4,865 & 425 & 4,430 & 860 & 4,055 \\
[1ex]
Improvement \% & 36\% & 10\% (Insig) & 19\% & 24\% (Insig) & 16\% & 32\% (Insig) & 32\% \\
\bottomrule
\end{tabular}
}
\begin{tablenotes}\footnotesize
\item \textit{Note:} Standard errors in parentheses. $^{*}p < 0.05$, $^{**}p < 0.01$, $^{***}p < 0.001$.
\end{tablenotes}
\end{threeparttable}
\end{table}

\subsection{Robustness Checks}
\label{sup:robustnesschecks}

We performed a number of robustness checks to support our main results, summarized in Table~\ref{tab:all_appendix}. Specifically, we performed difference-in-differences (\S2.6.1), geographic matching (\S2.6.2), imputation strategies (\S2.6.3), differential missingness analysis (\S2.6.4), and a substitution-robust analysis (\S2.6.5). We also re-ran our main analysis using alternative controls (\S2.6.6), examined the treatment effect on compliers (\S2.6.7), and examined stockouts (\S2.6.8).

\begin{table}[t]
\centering
\caption{\textbf{Robustness Analysis Results.}}
\label{tab:all_appendix}
\begin{threeparttable}
\small
\begin{tabular}{l c c c c}
\toprule
\textbf{Model} & \textbf{Coefficient} & \textbf{Std. Error} & \textbf{Observations} & \textbf{Improvement \%} \\
\midrule
\multicolumn{5}{l}{\textit{Dependent Variable = Normalized Consumption}} \\
\midrule
SynthDiD & 0.116$^{*}$ & (0.046) & 5,290 & 19\% \\
DiD & 0.128$^{**}$ & (0.046) & 5,290 & 21\% \\
Matching (5km) & 0.154$^{*}$ & (0.066) & 1,125 & 25\% \\
Matching (10km) & 0.166$^{**}$ & (0.064) & 1,815 & 30\% \\
Matching (15km) & 0.121$^{*}$ & (0.054) & 2,415 & 21\% \\
Imputation (low rank) & 0.037$^{**}$ & (0.014) & 5,455 & 15\% \\
Imputation (avg consump) & 0.067$^{*}$ & (0.031) & 5,455 & 21\% \\
Imputation (pop) & 0.076$^{**}$ & (0.028) & 5,455 & 27\% \\
No Missingness Imbalance & 0.132$^*$ & (0.058) & 5,285 & 18\% \\
Substitution (Method 2) & 0.132$^**$ & (0.048) & 5,290 & 20\% \\
Alt. Control & 0.095$^{***}$ & (0.022) & 10,520 & 18\% \\
Substitution (Method 1) & 0.119$^*$ & (0.048) & 5,290 & 18\% \\
LATE (Continuous) & 0.115$^{**}$ & (0.038) & 5,290 & 37\% \\
\midrule
\multicolumn{5}{l}{\textit{Dependent Variable = Stockout}} \\
\midrule
SynthDiD & $-$0.280 & (0.179) & 5,290 & $-$4.6\% (Insig) \\
\midrule
\multicolumn{5}{l}{\textit{Dependent Variable = Missingness}} \\
\midrule
SynthDiD & $-$0.007 & (0.006) & 5,290 & $-$1.4\% (Insig) \\
\bottomrule
\end{tabular}
\begin{tablenotes}[flushleft]
\footnotesize
\item \textit{Notes:} $^{*} p < 0.05$, $^{**} p < 0.01$, $^{***} p < 0.001$. Standard errors in parentheses. 
The balanced dataset includes 1,058 facilities across five quarters from 2022 Q3 to 2023 Q3. The Improvement \% column reports the relative change in consumption compared to the counterfactual mean for each specification.
\end{tablenotes}
\end{threeparttable}
\label{tab:allr3}
\end{table}

\subsubsection{Difference-in-Differences (DiD).}

First, we ensured our results are consistent under a simpler DiD analysis~\cite{ashenfelter1984using}, which compares changes in outcomes over time between the treated and control groups. Fig.~\ref{fig:eventDiD} shows the corresponding event study, which shows that there are no statistically significant differences between treated and control facilities prior to our intervention. The DiD specification is:
\begin{equation}
Y_{n,t} = \mu + \alpha_n + \beta_t + \tau (\text{Treat}_n \times \text{Post}_t)
+ \epsilon_{n,t},
\end{equation}
where $Y_{n,t}$ is the observed outcome (i.e., normalized consumption) for unit $n$ at time $t$; $\mu$ denotes the mean outcome; $\alpha_n$ represents unit fixed effects; $\beta_t$ represents time period fixed effects;
$\epsilon_{n,t}$ is the error term; $\text{Treat}_n$ is a dummy indicator of treatment; $\text{Post}_t$ is a dummy indicator for post-treatment periods; and $\tau$ is the treatment effect. We find that the magnitude and significance of the increase in consumption are similar to those obtained using SynthDiD.

\begin{figure}[t]
\begin{center}
\includegraphics[width=0.7\linewidth]{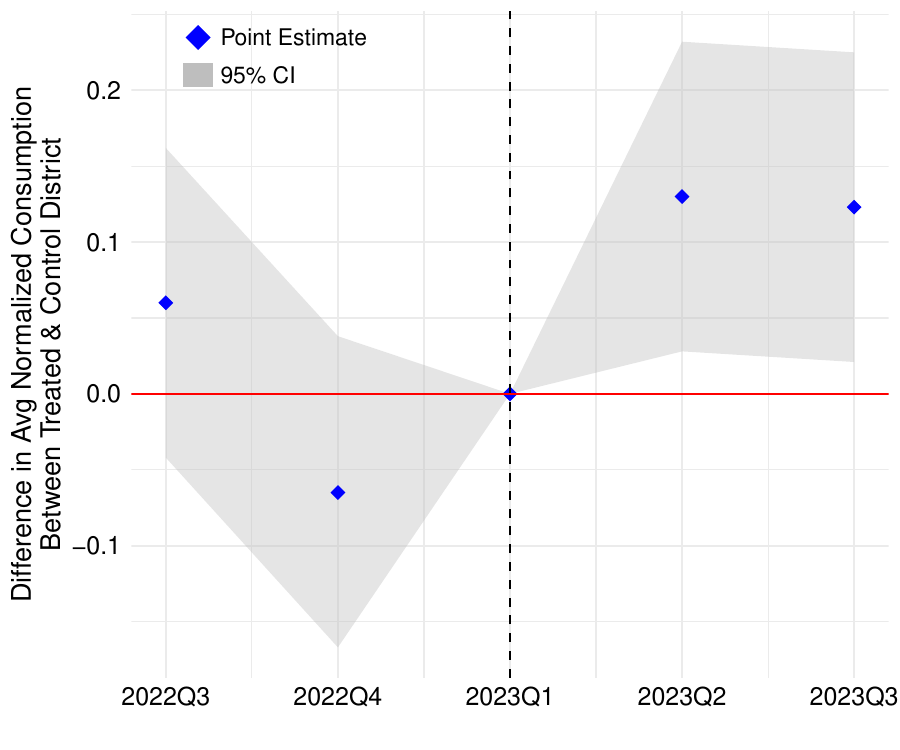}
\end{center}
\caption{\textbf{DiD Event Study.} This plot shows estimated treatment effects across time. Blue dots represent point estimates, and vertical bars denote 95\% confidence intervals. The black vertical dashed line shows the time period immediately preceding our 2023 Q2 deployment. As expected, the pre-treatment estimates are close to zero, while the post-treatment estimates are larger, indicating a positive impact on consumption attributable to our deployment.}
\label{fig:eventDiD}
\end{figure}

\subsubsection{Matching analysis.} Next, we employed a geographic matching design, which is more robust to geographically-varying unobserved confounders by restricting the DiD analysis to matching pairs of treatment and control facilities that are geographically close together. Specifically, we first restricted our sample to facilities within some distance of a district border (either 5km, 10km, or 15km) and that have a neighboring facility with the opposite treatment status; Fig.~\ref{fig:matchmap} shows the resulting set of facilities for the 15km cutoff. Then, we ran our DiD analysis restricted to matched facilities. Results for all three distances are shown in Table~\ref{tab:all_appendix}; they show a statistically significant increase in consumption for treated facilities that is consistent with our main analysis.

\begin{figure}[t]
\begin{center}
\includegraphics[width=0.5\linewidth]{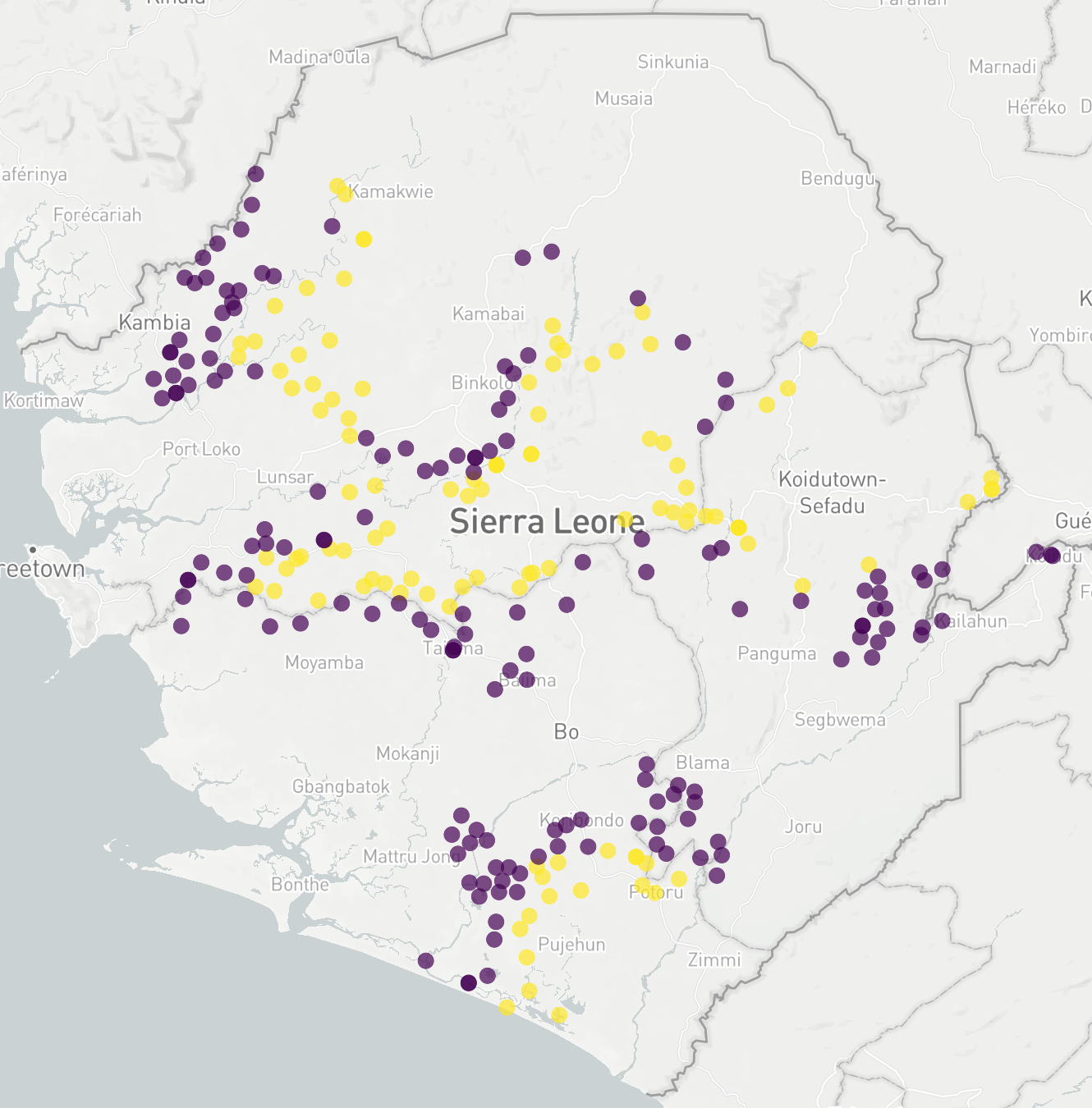}
\caption{\textbf{Map of Matched Border Facilities.} This map displays the geographic distribution of facilities included in the border matching analysis in the ``Matching (15km)'' row of Table~\ref{tab:all_appendix}. Purple dots represent control facilities and yellow dots represent treated facilities located within 15km of district borders. Shorter distances (5km, 10km) yield similar results.}
\label{fig:matchmap}
\end{center}
\end{figure}

\subsubsection{Imputation strategies.} Our main analysis simply dropped missing values; however, this may create concerns about bias from missing/unreliable data, and so we tested three different imputation strategies (based on low-rank completion, catchment population, and average demand) and found consistent treatment effects across all of them (see Table~\ref{tab:all_appendix}). Specifically:
\begin{itemize}
\item \textbf{Low-rank imputation (ImputedLowRank):} We used low-rank matrix completion~\cite{candes2012exact}, a standard approach for handling missing data in large matrices. We used the \textit{softImpute} R package with $\text{rank}=2$ and regularization parameter $\lambda=0.01$ (chosen using cross-validation) to impute missing consumption values. Table~\ref{tab:lambda_robustness} shows that our imputation results are robust to the choice of regularization parameter $\lambda$.
\item \textbf{Population-based imputation (ImputedPop):} We first estimated demand in proportion to each facility’s catchment population, and then imputed consumption as the minimum of the estimated demand and the computed allocation (we computed the allocation via the Excel tool in quarters prior to our tool's deployment, and using our tool otherwise).
\item \textbf{Average imputation (ImputedAvgConsump).} Third, using comprehensive data on unique quarter, facility, and product tuples, we imputed missing consumption values by assigning it to be the average quarterly consumption of the product across all facilities.
\end{itemize}
We found that our estimates are qualitatively similar to our main analysis, suggesting that our results are robust to different imputation strategies.

\begin{table}[t]
\caption{\textbf{Low-Rank Imputation with Different $\lambda$ Values.}}
\label{tab:lambda_robustness}
\centering
\begin{threeparttable}
\begin{tabular}{lcccc}
\toprule
 & \textbf{Coefficient} & \textbf{Std. Error} & \textbf{Observations} & \textbf{Improvement \%} \\
\midrule
\multicolumn{5}{l}{\textit{Dependent Variable = Normalized Consumption}} \\
\midrule
(1) $\lambda = 0.01$ & 0.037$^{**}$ & (0.014) & 5,455 & 15\% \\
(2) $\lambda = 0.1$  & 0.044$^{***}$ & (0.013) & 5,455 & 16\% \\
(3) $\lambda = 1$    & 0.035$^{**}$ & (0.012) & 5,455 & 13\% \\
(4) $\lambda = 5$    & 0.042$^{***}$ & (0.012) & 5,455 & 14\% \\
(5) $\lambda = 10$   & 0.039$^{**}$ & (0.012) & 5,455 & 16\% \\
\bottomrule
\end{tabular}
\begin{tablenotes}
\footnotesize
\item Notes: Standard errors in parentheses. \textsuperscript{*}$p < 0.05$, \textsuperscript{**}$p < 0.01$, \textsuperscript{***}$p < 0.001$. 
\end{tablenotes}
\end{threeparttable}
\end{table}

\subsubsection{Missingness analysis.} \label{ssec:missingness}

We checked that (1) our allocation tool did not drive differential missingness patterns, and (2) that products with differential missingness did not drive our results:
\begin{enumerate}
\item Missingness as the outcome variable: We used a missingness indicator as the dependent variable in our main specification; results are shown in the ``Missingness'' row of Table~\ref{tab:all_appendix}, and show no significant treatment effect. These results suggest that our intervention did not cause differential missing data, which could bias our results.
\item Analysis on a sub-sample with no missingness imbalance: We also re-ran our main analysis on a subset of products that showed no significant difference in missing data rates between the treatment and control groups after the treatment. Results are shown in the ``No Missingness Imbalance'' row of Table~\ref{tab:all_appendix}; they are consistent with our main analysis, showing a statistically significant 18\% increase in consumption.
\end{enumerate}
These results suggest that our findings are not driven by differential missing data patterns.

\subsubsection{Substitution.}
As discussed in \S\ref{sup:substitution}, facilities may substitute related products when stock of one product is low. As outlined in \S\ref{sup:substitution}, we merged substitutable products into clusters, and performed our analysis on (appropriately normalized) consumption of product clusters. 

To this end, we developed groupings of substitutable products based on conversations with two medical experts, and validated them through discussions with field workers. Table~\ref{tab:sublist} shows substitutable groups of products among the ones we allocated; direct substitutes are perfect, functional substitutes are nearly perfect, and therapeutic substitutes are more complicated but can reasonably be approximated as being perfect. Our discussions suggested that group 5 was context-dependent, so we performed two analyses---one with all seven groups (Method 1) and one omitting group 5 (Method 2). We re-ran our primary analysis using both groupings, which accounts for substitution effects according to our theoretical analysis. Our results (the ``Substitution'' rows in Table~\ref{tab:all_appendix}) show consistent and significant ATTs of 18\% or 20\%, suggesting that substitution behavior is unlikely to significantly affect our findings.

\begin{table}[t]
\centering \small
\caption{\textbf{Medicine Substitutions.} This table lists clinically similar products grouped by substitution type: Direct Substitutions (same ingredient but in different doses or forms), Functional Substitutions (different products with similar clinical purposes), and Therapeutic Substitutions (different medications for the same condition). These expert-provided substitution groups were used to test whether observed consumption increases could be driven by product substitution. From our conversations, substitution group (5) (marked with a $\dagger$) was context-dependent; thus, we performed checks both with and without this group.}
\begin{tabular}{l p{10cm}}
\toprule
Type of Substitution & Substitutes \\
\midrule
Direct Substitutions & (1) Paracetamol (Acetaminophen) 500mg, Tab; Paracetamol (Acetaminophen) 250mg, Dispersible, Tab \\
\addlinespace
Functional Substitutions & (2) Syringe, Luer, 2ml, Disposable, Pcs; Syringe, Luer, 10ml, Disposable, Pcs; Syringe, Luer, 20ml, Disposable, Pcs \\ & (3) Glucose (Dextrose) 5\%, IV Inj, 500ml, Soft Bag; Glucose (Dextrose) Hypertonic 50\%, IV Inj, 50ml, Bot \\ & (4) Cannula, IV, 18G, Short, Sterile, Disposable, Pcs; Cannula, IV, 24G, Short, Sterile, Disposable, Pcs \\
\addlinespace
Therapeutic Substitutions & (5)$^\dagger$ Neomycin \& Bacitracin 0.5\% \& 500IU/g, Ointment, 15g, Tube; Povidone Iodine 10\%, Solution, Bot; Chlorhexidine Gluconate 7.1\%, Gel \\ & (6) Oxytocin 10IU, Inj, Amp; Misoprostol 200mcg, Tab \\ & (7) Ampicillin 500mg, Pdr for IM/IV, Inj, Vial; Ceftriaxone 250mg, Pdr for Inj \\
\midrule
\multicolumn{2}{p{15.5cm}}{\small{\textbf{Note:} Every substitution requires a qualified healthcare professional to first assess the patient's specific diagnosis and clinical context, then manage the substitution by adjusting dosage and following clinical or administration procedures (e.g., dilution, etc.).}} \\
\bottomrule
\end{tabular}
\label{tab:sublist}
\end{table}

\subsubsection{Alternative control group (Alt. Control).}
We also performed our analysis using an alternative control group of 26 products that were concurrently allocated using a different, pre-existing mechanism (see Table~\ref{tab:medicine_consumption_treat0}). The consumption levels for these products can be used as a control group throughout our study period across all districts in the country using a staggered treatment---i.e., an advantage of this analysis is that it can be performed not just for the partial deployment in 2023 Q2, but also for the nationwide implementation starting in Q3. We used the \textit{sdid} package in STATA~\cite{clarke2023synthetic}, and found a consistent and statistically significant increase in consumption of 18\% (``Alt. Control'' row in Table~\ref{tab:all_appendix}).

\subsubsection{Local average treatment effect.} \label{ssec:LATE}

To disentangle the effect of compliance with our recommendations from the impact of our system, we performed a standard instrumental variables analysis to compute the Local Average Treatment Effect (LATE)~\cite{angrist1995identification}. Following standard practice, we used the government's random (binary) treatment assignment as the instrument, and compliance (measured as a continuous variable, described in \S\ref{sup:compliance}) as the treatment variable.

In general, IV analyses require two conditions to be satisfied. The first is relevance---our first-stage results in Table~\ref{tab:IVC} show a strong correlation between treatment assignment and compliance ($\text{coefficient} = 0.89$, $p < 0.001$), suggesting that the relevance assumption holds. The second is the exclusion restriction, which requires that the assignment to treatment affects outcomes only through compliance. This is highly plausible in our case. For consumption to change, either patient/provider behavior or supply chain logistics must change. Yet, patients and facility-level providers had no knowledge of our deployment (or more generally, NMSA's operations), and assignment alone did not alter logistics unless the district actually adopted our recommended allocation. Thus, we believe assignment to treatment would not influence consumption except through compliance, satisfying the exclusion restriction.

To estimate the LATE, we used Two-Stage Least Squares (2SLS), with the first-stage
\begin{equation}
\label{eq:first_stage}
D_n = \pi_1 Z_n + \pi_2 X_n + \mu_n,
\end{equation}
where $D_n$ is compliance, $Z_n$ is treatment assignment, and $X_n$ denotes covariates (indicators for the district, quarter, and facility type). Here, $n$ ranges over all (treated and control) facilities.

The second-stage regression is given by
\begin{equation} \label{eq:second_stage}
Y_n = \beta_1 D_n + \beta_2 X_n + \epsilon_n,
\end{equation}
where $Y_n$ is the outcome variable (i.e., normalized consumption). Table~\ref{tab:IVC} shows results, yielding a significant increase of 0.12 standard deviations. To interpret this result, we make an additional assumption that the first-stage relationship between assignment and compliance is monotonic and linear---then, this result can be interpreted as ``a one-unit increase in compliance induced by the instrument yields a 0.12 standard deviation increase in outcome,'' which can be translated to a 37\% reduction in unmet demand among compliers.\footnote{We can recover the IV coefficient in percentage terms: The implied percentage increase in consumption for each product $p$ is computed as $\%\Delta^{\text{LATE}}_p = 100 \times \hat{\beta}^{\text{norm}} \times \frac{s_p}{\bar{Y}_p}$, where $\hat{\beta}^{\text{norm}}$ is the 2SLS coefficient on the standardized outcome, $s_p$ is the product-specific standard deviation, and $\bar{Y}_p$ is its baseline mean consumption. Then, we take the simple average of these product-level effects to obtain an overall increase of approximately 37\% for compliers.}

\begin{table}[t]
\centering
\caption{\textbf{LATE IV Result.} Column 1 shows the OLS estimate of treatment compliance (continuous measure defined in Supplement \S\ref{sup:compliance}) on normalized consumption. Column 2 presents the first-stage regression of treatment assignment (instrument) on compliance. Column 3 reports the IV estimate, which identifies the causal effect for compliers.}
\label{tab:IVC}
\begin{tabular}{l*{3}{r}}
\hline\hline
&\multicolumn{1}{c}{\parbox{2.5cm}{\centering (1)\\OLS}}&\multicolumn{1}{c}{\parbox{2.5cm}{\centering (2)\\First-Stage}}&\multicolumn{1}{c}{\parbox{2.5cm}{\centering (3)\\IV}}\\
\cline{2-4}
Dependent variable: & \parbox{2.5cm}{\centering Normalized \\ Consumption} & \parbox{2.5cm}{\centering Treated Complier} & \parbox{2.5cm}{\centering Normalized \\ Consumption} \\
\hline
Treated Complier                      & 0.075** & \multicolumn{1}{r}{--}      & \textbf{0.115**}   \\
& (0.023) & & (0.038) \\
IV (treatment assignment)             & \multicolumn{1}{r}{--}    & \multicolumn{1}{r}{0.887***}      & \multicolumn{1}{r}{--}    \\
  & & (0.016) & \\
\hline
Facility fixed effect                 & YES                       & YES                         & YES                 \\
Quarter fixed effect                  & YES                       & YES                         & YES                 \\
Facility Type fixed effect            & YES                       & YES                         & YES                 \\
\hline
Observations                          & \multicolumn{1}{r}{5,290} & \multicolumn{1}{r}{5,290} & \multicolumn{1}{r}{5,290} \\
\(R^2\)                               & \multicolumn{1}{r}{0.067}   & \multicolumn{1}{r}{-}      & \multicolumn{1}{r}{0.067}   \\
First-Stage \(F\)-Stat                & \multicolumn{1}{r}{--}    & \multicolumn{1}{r}{3,318}  & \multicolumn{1}{r}{--}      \\
\hline\hline
\multicolumn{4}{l}{\textit{Notes: Standard errors in parentheses. \textsuperscript{*}p$<$0.05, \textsuperscript{**}p$<$0.01, \textsuperscript{***}p$<$0.001.}} \\
\end{tabular}
\end{table}

\subsubsection{Stockouts.} We also assessed the impact of our system on the number of facility-product stockouts. While reducing stockouts might seem like a natural objective, optimizing for fewer stockouts produces highly undesirable allocations---e.g., an optimal strategy is to allocate zero supply to a small number of high-volume facilities, thereby ensuring that the remaining facilities are well-stocked. Using SynthDiD but replacing our unit-level outcomes with binary stockout indicators, we found a directional reduction of 4.6\% in stockouts but it is not statistically significant ($p\approx0.12$) (``Stockouts'' rows in Table~\ref{tab:all_appendix})---i.e., our system does not inadvertently increase stockouts.

\subsection{Cost-Effectiveness Analysis}
\label{sup:cost}

Our system was highly cost-effective. By design, it did not induce new costs for warehousing or logistics; instead, it integrated directly into the existing supply chain infrastructure, replacing the previous manual planning process without requiring additional physical resources or personnel. The primary cost was a nationwide \$30 USD monthly server fee for hosting the prediction and optimization pipeline and the interface by which NMSA workers interact with our system.

We conducted an Incremental Cost-Effectiveness Ratio (ICER) analysis using Disability-Adjusted Life Years (DALYs) as our measure of effectiveness; one DALY represents one year of healthy life, with one year of life with some disease, disability, or other health condition represented by some fraction less than one. To this end, we matched each medicine in our study to the condition it treats, and obtained the number of DALYs gained by treating that condition according to the World Health Organization~\cite{WHO_2024_GlobalHealthEstimates}. For products that can treat multiple conditions, we used a conservative approach by using the minimum DALY value that one unit of the medicine could avert; this strategy ensures that our effectiveness estimate is a lower bound. The ICER is then calculated as the incremental cost of the intervention $\Delta C$ divided by the incremental health effect $\Delta E$. For every product $m\in[M]$, let
$\bar{q}_m$ denote the minimum number of DALYs that can be saved when one extra unit of product $m$ is consumed, and let $F=\text{360 USD}$ denote the yearly fixed cost of hosting our system for the entire nation. Then, we have
\begin{align*}
\mathrm{ICER}
=
\frac{\Delta C}{\Delta E} = 
\frac{F}{\delta\sum_{m=1}^M\bar{q}_m},
\end{align*}
where $\delta$ is the treatment effect from our main impact evaluation. We used $\delta = 0.19$, representing the 19\% average increase in medicine consumption attributed to our system, and we computed the yearly increase in DALYs from increased consumption of distributed products $\sum_{m=1}^M\bar{q}_m=361.3$. This analysis yields an average ICER of only \$5.24 USD per DALY. This is much smaller than the ICER of typical interventions that are considered cost-effective 
(\$50{,}000--\$100{,}000)~\cite{ICER2019ThresholdRanges,HERC_CEA,neumann2014updating}. While our ICER is extremely low, it is consistent with recent literature on the ``unreasonable effectiveness of algorithms,'' which finds that algorithms can significantly improve resource allocation at a very low marginal cost~\cite{ludwig2024unreasonable}.

\begin{table}
\footnotesize
\centering
\caption{\textbf{Products Allocated via Our Tool.} Statistics for monthly facility-level consumption between Jan 2020 and Nov 2023. Missing values are excluded.}
\label{tab:medicine_consumption_treat1}
\begin{tabular}{p{8cm} r r r}
\toprule
\textbf{Product Name} & \textbf{Mean} & \textbf{Median} & \textbf{Std. Dev.} \\
\midrule

\multicolumn{4}{l}{\textbf{Medicines}\hfill} \\
\midrule
Albendazole 400mg, Tab & 124.89 & 90 & 205.83 \\
Aluminium Hydroxide 500mg, Tab & 169.70 & 85 & 246.70 \\
Ampicillin 500mg, Pdr for IM/IV, Inj, Vial & 65.55 & 34 & 203.59 \\
Benzyl Benzoate 25\%, Emulsion, 100ml, Bot & 12.76 & 1 & 65.00 \\
Ceftriaxone 250mg, Pdr for Inj & 113.39 & 25 & 226.94 \\
Chlorhexidine Gluconate 7.1\%, Gel & 3.58 & 2 & 16.11 \\
Ciprofloxacin 500mg, Tab & 171.57 & 100 & 268.18 \\
Erythromycin 125mg/5ml, Pdr for Susp, 100ml, Bot & 89.27 & 15 & 224.83 \\
Ferrous Sulphate 200mg, Tab & 512.29 & 300 & 673.05 \\
Folic Acid 5mg, Tab & 509.04 & 398 & 745.81 \\
Glucose (Dextrose) 5\%, IV Inj, 500ml, Soft Bag & 40.73 & 6 & 155.03 \\
Glucose (Dextrose) Hypertonic 50\%, IV Inj, 50ml, Bot & 9.60 & 2 & 27.87 \\
Methyldopa 250mg, Tab & 73.73 & 45 & 168.29 \\
Metoclopramide HCl 10mg, Tab & 147.31 & 30 & 306.85 \\
Misoprostol 200mcg, Tab & 31.91 & 10 & 144.94 \\
Neomycin \& Bacitracin 0.5\% \& 500IU/g, Ointment, 15g, Tube & 4.70 & 2 & 17.82 \\
Oral Rehydration Salts (ORS), Sachet & 121.27 & 100 & 187.75 \\
Oxytocin 10IU, Inj, Amp & 19.61 & 10 & 86.72 \\
Paracetamol (Acetaminophen) 250mg, Dispersible, Tab & 464.35 & 400 & 475.49 \\
Paracetamol (Acetaminophen) 500mg, Tab & 562.41 & 455 & 635.66 \\
Povidone Iodine 10\%, Solution, Bot & 8.64 & 1 & 84.45 \\
Prednisolone 5mg, Tab & 65.24 & 12 & 134.54 \\
Salbutamol 100mcg/dose, Aerosol, Inhaler & 1.50 & 0 & 14.28 \\
Water for Injection 10ml, Inj, Amp & 33.75 & 20 & 124.35 \\
Zinc Sulphate 20mg, Dispersible, Tab & 221.20 & 130 & 459.66 \\

\midrule
\multicolumn{4}{l}{\textbf{Medical Supplies \& Equipment}\hfill} \\
\midrule
Apron, Plastic, Disposable, Pcs & 24.69 & 10 & 124.65 \\
Cannula, IV, 18G, Short, Sterile, Disposable, Pcs & 14.90 & 10 & 48.70 \\
Cannula, IV, 24G, Short, Sterile, Disposable, Pcs & 16.24 & 10 & 61.55 \\
Envelope, Dispensing, Plastic, 10cm x 7cm, Pcs & 150.11 & 100 & 184.44 \\
Glove, Exam, Latex, Medium, Nonsterile, Disposable, Pcs & 173.18 & 100 & 338.22 \\
IV Giving Set, Pcs & 19.23 & 11 & 49.27 \\
Mask, Surgical, Pcs & 47.73 & 12 & 215.93 \\
Rapid test kit, Pregnancy, Pcs & 13.73 & 10 & 38.64 \\
Syringe, Luer, 10ml, Disposable, Pcs & 37.65 & 17 & 188.69 \\
Syringe, Luer, 20ml, Disposable, Pcs & 25.80 & 10 & 75.47 \\
Syringe, Luer, 2ml, Disposable, Pcs & 41.23 & 25 & 97.43 \\
\bottomrule
\end{tabular}
\end{table}

\begin{table}
\footnotesize
\centering
\caption{\textbf{Products Allocated via Other Mechanisms.} Statistics for monthly facility-level consumption between Jan 2020 and Nov 2023. Missing values are excluded. Data from these products were only used to help train our prediction model and for one of our evaluations (staggered SynthDiD using product-level controls, called ``Alt. Control'' in Table~\ref{tab:att_tabular}).}
\label{tab:medicine_consumption_treat0}
\begin{tabular}{p{8cm} r r r}
\toprule
\textbf{Product Name} & \textbf{Mean} & \textbf{Median} & \textbf{Std. Dev.} \\
\midrule
\multicolumn{4}{l}{\textbf{Medicines}\hfill} \\
\midrule
Amoxicillin 250mg, Dispersible, Tab & 592.01 & 456 & 734.22 \\
Calcium Gluconate 100mg/ml, Inj, 10ml, Amp & 5.20 & 0 & 52.47 \\
Ceftriaxone 1g, Pdr for Inj, Vial & 164.45 & 20 & 502.68 \\
Chlorhexidine Gluconate 5\%, Solution, 1000ml, Bot & 6.38 & 1 & 116.60 \\
Cloxacillin 500mg, Tab/Cap & 300.74 & 100 & 550.75 \\
Epinephrine HCl (Adrenaline) 1mg/ml, Inj, 1ml, Amp & 9.25 & 0 & 78.06 \\
Ibuprofen 400mg, Tab & 284.31 & 200 & 346.74 \\
Jadelle & 8.07 & 5 & 22.51 \\
Levonorgestrel (Emergency Contraceptive) 1.5mg, Tab & 4.87 & 0 & 15.75 \\
Levoplant & 5.33 & 1 & 17.90 \\
Nystatin 500,000IU, Tab & 8.74 & 4 & 24.61 \\
Progesterone-Only (Microlut) Levonorgestrel 30mcg, Tab, Cycle & 19.63 & 3 & 145.70 \\
Sodium Chloride (Normal Saline) 0.9\%, IV Inj, 500ml, Bot & 12.27 & 8 & 50.45 \\

\midrule
\multicolumn{4}{l}{\textbf{Medical Supplies \& Equipment}\hfill} \\
\midrule
Bandage, Elastic, 8cm x 4m, Roll & 10.75 & 1 & 43.08 \\
Blade, Surgical, No. 22, Sterile, Disposable, Pcs & 18.23 & 5 & 91.90 \\
Cannula, IV, 22G, Short, Sterile, Disposable, Pcs & 14.33 & 6 & 51.38 \\
Condom, Male & 114.40 & 60 & 258.42 \\
Copper-Containing Device (Copper T or Copper 7 or IUD) & 7.41 & 0 & 28.42 \\
Cotton Wool, Absorbent, 500g, Roll & 2.73 & 1 & 17.93 \\
Glove, Surgical, Size 7.5, Sterile, Disposable, Pair & 49.34 & 13 & 112.43 \\
Glove, Surgical, Size 8, Sterile, Disposable, Pair & 37.98 & 10 & 108.17 \\
Needle, Hypodermic, Luer, 21G, Sterile, Disposable & 69.56 & 50 & 135.77 \\
Needle, Hypodermic, Luer, 23G, Sterile, Disposable & 62.09 & 50 & 97.86 \\
Sanitary Pads, Pcs & 9.56 & 2 & 20.52 \\
Syringe, Luer, 5ml, Disposable & 57.06 & 45 & 111.68 \\
Tube, Asp/Feed, CH12, Sterile, Disposable, Pcs & 4.16 & 0 & 19.59 \\
\bottomrule
\end{tabular}
\end{table}

\begin{table}[t]
\caption{\textbf{ Balance Table: District Characteristics by Treatment Status}}
\label{tab:balance_district}
\centering
\resizebox{\textwidth}{!}{
\begin{tabular}{lcccccc}
\toprule
& \multicolumn{2}{c}{Control} & \multicolumn{2}{c}{Treatment} & \multicolumn{2}{c}{Difference} \\
\cmidrule(lr){2-3} \cmidrule(lr){4-5} \cmidrule(lr){6-7}
Variable & Mean & (SD) & Mean & (SD) & p-value & \begin{tabular}[c]{@{}c@{}}FDR Adj. \\ p-value\end{tabular} \\
\midrule
Days of service delivery & 0.46 & (0.17) & 0.38 & (0.17) & 0.35 & 0.71 \\
Hours of service delivery & 1.40 & (0.56) & 1.26 & (0.55) & 0.14 & 0.71 \\
Facilities where women give birth & 6.44 & (2.54) & 5.23 & (2.47) & 0.19 & 0.71 \\
Availability of basic emergency obstetric and neonatal care & 0.43 & (0.38) & 0.19 & (0.18) & 0.35 & 0.71 \\
Availability of priority drugs & 3.62 & (1.39) & 3.12 & (1.58) & 0.78 & 0.88 \\
Availability of all tracer drugs & 1.84 & (1.41) & 2.40 & (1.78) & 0.24 & 0.71 \\
Availability of vaccines & 6.32 & (2.31) & 5.31 & (2.46) & 0.68 & 0.88 \\
Vaccines storage: refrigerators $2$--$8\,^{\circ}\mathrm{C}$ & 1.72 & (2.72) & 1.06 & (2.14) & 0.87 & 0.88 \\
Availability of communication equipment & 3.65 & (2.56) & 2.56 & (1.50) & 0.41 & 0.71 \\
Access to various forms of communication & 3.65 & (2.56) & 2.56 & (1.50) & 0.41 & 0.71 \\
Total proportion of facilities carrying out safe health care waste disposal & 5.00 & (2.10) & 3.86 & (1.66) & 0.68 & 0.88 \\
Availability of basic equipment & 1.99 & (0.67) & 2.00 & (0.91) & 0.31 & 0.71 \\
Availability of Standard Treatment Guidelines & 3.63 & (1.97) & 2.53 & (2.56) & 0.44 & 0.71 \\
Outpatient caseload (median per facility) & 0.63 & (0.52) & 0.51 & (0.21) & 0.88 & 0.88 \\
Facilities with community health workers & 5.44 & (2.54) & 4.84 & (2.82) & 0.78 & 0.88 \\
Average number of community health workers & 0.69 & (0.24) & 0.64 & (0.33) & 0.57 & 0.85 \\
Average health workers per facility & 0.47 & (0.50) & 0.41 & (0.38) & 0.81 & 0.88 \\
Population Share & 6.95 & (3.13) & 4.72 & (2.45) & 0.15 & 0.71 \\
\hline
Human Development Index (HDI) & 0.45 & (0.05) & 0.43 & (0.02) & 0.16 & 0.21 \\
Health Index (health dimension of HDI based on life expectancy at birth) & 0.63 & (0.05) & 0.64 & (0.06) & 0.66 & 0.66\\
Income Index (income dimension of HDI based on Gross National Income per capita) & 0.41 & (0.06) & 0.38 & (0.02) & 0.11 & 0.21\\
\bottomrule
\end{tabular}
}
\begin{minipage}{\textwidth}
\footnotesize
\textit{Notes:} Means and standard deviations are shown as Mean (SD). p-values are from two-sample Welch’s t-tests. Multiple testing adjustments control the false discovery rate (FDR) using the Benjamini–Hochberg (BH) procedure.
\end{minipage}
\end{table}

\begin{table}
\scriptsize
\centering
\caption{\textbf{Percentage of Missing Data by Product.} We report missingness rates in consumption data for each product across control and treated facilities in pre- and post-treatment periods. The $p$-values are from  two-sample proportion tests of significant differences in missingness rates between control and treated groups, without adjusting for multiple hypothesis testing.}
\label{tab:balanceMiss}
\begin{tabular}{
>{\raggedright\arraybackslash}p{5.5cm}
>{\centering\arraybackslash}p{1.2cm}
>{\centering\arraybackslash}p{1.2cm}
>{\centering\arraybackslash}p{1.2cm}
>{\centering\arraybackslash}p{1.2cm}
>{\centering\arraybackslash}p{1.2cm}
>{\centering\arraybackslash}p{1.2cm}
}
\toprule
& \multicolumn{3}{c}{\textbf{Pre-Treated}} & \multicolumn{3}{c}{\textbf{Post-Treated}} \\
\cmidrule(l{2pt}r{2pt}){2-4} \cmidrule(l{2pt}r{2pt}){5-7}
\textbf{Product} &
\makecell{\small Control \\ \small (\%)} &
\makecell{\small Treated \\ \small (\%)} &
\makecell{\small \textit{p}-value} &
\makecell{\small Control \\ \small (\%)} &
\makecell{\small Treated \\ \small (\%)} &
\makecell{\small \textit{p}-value} \\
\midrule
Albendazole 400mg, Tab & 0.00 & 0.00 & 1.00 & 0.10 & 0.00 & 1.00 \\
Aluminium Hydroxide 500mg, Tab & 97.60 & 99.00 & 0.15 & 95.20 & 91.70 & 1.00 \\
Ampicillin 500mg, Pdr for IM/IV, Inj, Vial & 14.60 & 24.20 & 0.06 & 5.50 & 10.10 & 0.00\textsuperscript{***} \\
Apron, Plastic, Disposable, Pcs & 73.80 & 42.60 & 0.00\textsuperscript{***} & 75.90 & 55.20 & 0.00\textsuperscript{***} \\
Benzyl Benzoate 25\%, Emulsion, 100ml, Bot & 98.50 & 98.70 & 1.00 & 70.20 & 83.80 & 0.00\textsuperscript{***} \\
Cannula, IV, 18G, Short, Sterile, Disposable, Pcs & 42.50 & 32.30 & 0.02\textsuperscript{*} & 66.40 & 53.60 & 0.00\textsuperscript{***} \\
Cannula, IV, 24G, Short, Sterile, Disposable, Pcs & 39.30 & 26.40 & 0.00\textsuperscript{***} & 6.50 & 3.60 & 1.00 \\
Ceftriaxone 250mg, Pdr for Inj & 96.70 & 99.40 & 0.66 & 88.90 & 88.20 & 0.75 \\
Chlorhexidine Gluconate 7.1\%, Gel & 28.70 & 15.90 & 0.00\textsuperscript{***} & 3.90 & 4.80 & 0.48 \\
Ciprofloxacin 500mg, Tab & 90.60 & 87.90 & 0.19 & 67.90 & 22.30 & 0.00\textsuperscript{***} \\
Envelope, Dispensing, Plastic, 10cm x 7cm, Pcs & 55.40 & 32.70 & 0.00\textsuperscript{***} & 70.90 & 52.90 & 0.00\textsuperscript{***} \\
Erythromycin 125mg/5ml, Pdr for Susp, 100ml, Bot & 99.50 & 99.00 & 0.43 & 82.70 & 85.40 & 0.32 \\
Ferrous Sulphate 200mg, Tab & 75.70 & 76.90 & 0.71 & 83.40 & 54.20 & 0.00\textsuperscript{***} \\
Folic Acid 5mg, Tab & 0.10 & 0.30 & 0.15 & 1.80 & 6.10 & 0.00\textsuperscript{***} \\
Glove, Exam, Latex, Medium, Nonsterile, Disposable, Pcs & 4.70 & 6.90 & 0.04\textsuperscript{*} & 6.90 & 4.30 & 0.05\\
Glucose (Dextrose) 5\%, IV Inj, 500ml, Soft Bag & 91.20 & 96.50 & 0.13 & 76.10 & 77.60 & 0.65 \\
Glucose (Dextrose) Hypertonic 50\%, IV Inj, 50ml, Bot & 99.60 & 99.40 & 0.64 & 87.00 & 92.90 & 0.38 \\
IV Giving Set, Pcs & 5.00 & 4.20 & 0.46 & 4.30 & 4.80 & 0.69 \\
Mask, Surgical, Pcs & 79.10 & 74.00 & 1.00 & 83.70 & 77.20 & 1.00 \\
Methyldopa 250mg, Tab & 0.20 & 0.50 & 0.30 & 0.30 & 0.30 & 1.00 \\
Metoclopramide HCl 10mg, Tab & 90.10 & 88.00 & 0.33 & 85.10 & 84.00 & 0.65 \\
Misoprostol 200mcg, Tab & 60.60 & 67.90 & 1.00 & 70.70 & 75.40 & 0.10 \\
Neomycin \& Bacitracin 0.5\% \& 500IU/g, Ointment, 15g, Tube & 97.60 & 99.00 & 0.15 & 69.30 & 68.70 & 0.83 \\
Oral Rehydration Salts (ORS), Sachet & 0.00 & 0.00 & 1.00 & 0.20 & 0.80 & 0.10 \\
Oxytocin 10IU, Inj, Amp & 1.50 & 1.30 & 0.86 & 3.30 & 3.70 & 0.68 \\
Paracetamol (Acetaminophen) 250mg, Dispersible, Tab & 0.20 & 0.20 & 0.64 & 0.10 & 0.20 & 1.00 \\
Paracetamol (Acetaminophen) 500mg, Tab & 38.20 & 28.40 & 0.01\textsuperscript{*} & 12.10 & 21.00 & 0.00\textsuperscript{**} \\
Povidone Iodine 10\%, Solution, Bot & 89.70 & 90.00 & 0.91 & 5.50 & 20.10 & 0.00\textsuperscript{***} \\
Prednisolone 5mg, Tab & 100.00 & 99.00 & 0.55 & 93.60 & 90.70 & 0.12 \\
Rapid test kit, Pregnancy, Pcs & 69.80 & 61.60 & 0.49 & 62.30 & 59.20 & 0.35 \\
Salbutamol 100mcg/dose, Aerosol, Inhaler & 89.10 & 76.80 & 0.00\textsuperscript{***} & 40.30 & 44.60 & 0.20 \\
Syringe, Luer, 10ml, Disposable, Pcs & 80.10 & 54.20 & 0.00\textsuperscript{***} & 57.90 & 29.00 & 0.00\textsuperscript{***} \\
Syringe, Luer, 20ml, Disposable, Pcs & 70.40 & 46.00 & 0.00\textsuperscript{***} & 61.50 & 31.80 & 0.00\textsuperscript{***} \\
Syringe, Luer, 2ml, Disposable, Pcs & 75.00 & 47.00 & 0.00\textsuperscript{***} & 56.80 & 31.20 & 0.00\textsuperscript{***} \\
Water for Injection 10ml, Inj, Amp & 25.40 & 25.20 & 0.95 & 32.30 & 30.80 & 0.68 \\
\addlinespace
Zinc Sulphate 20mg, Dispersible, Tab & 0.00 & 0.00 & 1.00 & 0.20 & 0.30 & 0.64 \\
\bottomrule
\end{tabular}
\begin{minipage}{\textwidth}
\footnotesize
\textit{Notes:} Significance levels: \textsuperscript{*}p$<$0.05, \textsuperscript{**}p$<$0.01, \textsuperscript{***}p$<$0.001.
\end{minipage}
\end{table}

\end{document}